\documentclass{article}

\usepackage[final,nonatbib]{nips_2016}

\usepackage[utf8]{inputenc} 
\usepackage[T1]{fontenc}    
\usepackage{hyperref}       
\usepackage{url}            
\usepackage{booktabs}       
\usepackage{amsfonts}       
\usepackage{nicefrac}       
\usepackage{microtype}      
\usepackage{amssymb}
\usepackage{algorithm}
\usepackage[titletoc]{appendix}

\usepackage[scaled=0.92]{helvet}   

\usepackage[T1]{fontenc}
\usepackage{amsmath}
\usepackage{amsfonts}
\usepackage{amsthm}
\usepackage{mathtools}
\usepackage{algpseudocode}
\usepackage{tikz}
\usepackage[mathscr]{euscript}
\usepackage{pdflscape}

\newtheorem{theorem}{Theorem}
\newtheorem{corollary}{Corollary}
\newtheorem{proposition}{Proposition}
\newtheorem{claim}{Claim}
\newtheorem{lemma}{Lemma}

\newcommand{\Expof}{\mathbb{E}}
\newcommand{\ELBO}{\mathcal{L}}
\newcommand{\DOM}{\mathcal{D}}
\newcommand{\betahat}{\hat{\beta}}
\newcommand{\pd}{\partial}
\newcommand\supp{\mathop{\rm supp}}

\usepackage{epsfig}
\usepackage{url}
\usepackage{bm}
\usepackage{bbm}

\title{Linking Sequences of Events with Sparse or \\No Common Occurrence across Data Sets}

\author{
	Yunsung Kim \\
	Department of Computer Science \\
	Columbia University \\
	\texttt{yunsung.kim@columbia.edu}
}

\begin{document}

\maketitle

\begin{abstract}

Data of practical interest - such as personal records, transaction logs, and medical histories - are sequential collections of events relevant to a particular source entity. Recent studies have attempted to link sequences that represent a common entity across data sets to allow more comprehensive statistical analyses and to identify potential privacy failures. Yet, current approaches remain tailored to their specific domains of application, and they fail when co-referent sequences in different data sets contain sparse or no common events, which occurs frequently in many cases.\\

To address this, we formalize the general problem of \emph{sequence linkage} and describe \emph{LDA-Link}, a generic solution that is applicable even when co-referent event sequences contain no common items at all. LDA-Link is built upon \emph{Split-Document} model, a new mixed-membership probabilistic model for the generation of event sequence collections. It detects the latent similarity of sequences and thus achieves robustness particularly when co-referent sequences share sparse or no event overlap. We apply LDA-Link in the context of social media profile reconciliation where users make no common posts across platforms, comparing to the state-of-the-art generic solution to sequence linkage.
\end{abstract}

\section{INTRODUCTION}
Given multiple independent data sources, it is desirable to link representations of identical entries to allow valuable statistical analyses or to save clerical efforts of identification. When the identifying attributes of the objects are absent, sanitized, or prone to error, resolving identity uncertainty becomes a highly non-trivial task. State-of-the-art solutions to identity uncertainty can effectively associate co-referent records across relational databases (record linkage) or dissimilar representations of the same object (author identification, noun coreference, image association), and they have widespread applications to reference matching \cite{lawrence1999autonomous, mccallum2000efficient, mccallum2004conditional, pasula2003identity}, public policy making \cite{sadinle2014detecting, sadinle2013generalized, jaro1989advances}, behavioral analysis \cite{ji2014structure, kazemi2015growing, narayanan2009anonymizing}, biomedical science \cite{christen2002febrl}, and database cleansing \cite{bhattacharya2007collective}.

Most data of practical interest today, however, is a collection in which each entry is a \emph{sequence} of events generated by an entity. This happens ubiquitously, especially when these entries reflect individual traits: shopping services keep track of different consumption histories, location-based applications maintain geo-tagged records of an individual’s whereabouts, and medical records and prescription histories contain a sequence of health-related incidents. Sequences collected in relevant domains reflect the common features of the same entity (e.g., consumption preferences, locational trajectories, or health conditions), which can provide evidence for matching ones that refer to a common source.

Although these sequences share an abstract generative pattern, their domain disparity prevents them from sharing rich common events. For example in many mobility-related applications, especially with granular records of time and location, one type of a spatiotemporal event (such as a phone call) is never guaranteed to occur in tandem with another (a credit card transaction). This poses a critical challenge for sequence linkage and also distinguishes the problem of sequence linkage from the previously studied problems of identity uncertainty. While the entries being matched in the latter consist of distortions or alternate forms of “ground-truth” features such as canonical names, objective measurements, or semantic definitions, no such ground-truth exists in sequence linkage so the solutions to these problems are crucially inapplicable. In addition, frequent abscence of common events makes sequence linkage an even more intricate task.

In this work, we present the simplest form of a generically applicable Bayesian framework that addresses the issues of rare common occurrence in cross-domain sequence linkage. This framework consists of a mixed-membership model for the generation of event sequences whose source entities are shared across data sets (Split-Document model), and a 3-phase unsupervised algorithm for inferring their identities across data sets (LDA-Link). Split-Document frames each event incident in terms of semantic “motifs,” and LDA-Link uses this characterization to determine the semantic similarities of a pair of views, ensuring robust linkage even when the co-referent sequence pairs share no common events at all. 

To validate the empirical robustness of LDA-Link against common occurrence sparsity, we also provide a case study with real-world geo-tagged data sets. Mobility-related data are the richest and the most omnipresent type of data available through numerous location-based smart phone applications and other external services such as social media, cellular logs, and credit card transaction histories. As diverse as these applications are, collections from one application are often independent from collections from another and their locations seldom coincide, which makes them a challenging target for sparsity-robust sequence linkage. LDA-Link achieves up to 37\% identifiability when linking profiles with no common posts across Instagram and Twitter, significantly outperforming the current state-of-the-art generic solution to sequence linkage.

The paper is organized as follows. Section~\ref{sec:relworks} presents an overview of relevant works and formalizes the problem of sequence linkage. Section~\ref{sec:lda_vb} reviews Latent Dirichlet Allocation and its inference methods that are essential to the development of our work. Section~\ref{sec:algo} presents Split-Document model and LDA-Link. In Section~\ref{sec:eval} we perform a case study of this method in the context of social media identity reconciliation, and validate its robustness against sparsity. We conclude our paper in Section~\ref{sec:concl}. The Appendix analyzes each step of the algorithm in theory and studies its convergence properties as well as its effectiveness. 

\section{RELATED WORKS AND PROBLEM FORMULATION}

\label{sec:relworks}

\subsection{PREVIOUS WORKS}

Sequence linkage belongs to the class of problems on identity uncertainty. Existing solutions to identity uncertainty are customized to linking representations of an entity with ``ground-truth'' field values such as canonical names, unambiguous semantics, numerical features, or objective measurements. Deterministic association can be done to declare exact matches on these features when error-free identifiers are available, but entries in large data sets tend to be prone to noise and distortions such as human inscription or measurement errors, use of alternative forms, subjective observations, or deliberate data sanitization.

In the absence of trustworthy identifiers, probabilistic methods can be used to address uncrontrollable noise. The seminal work of \cite{fellegi1969theory} provided the first probabilistic framework for linking records in relational databases that refer to the same entity (record linkage) based on agree/disagree match statuses of each field. Methods using generalized Expectation-Minimization \cite{winkler2000frequency, winkler1993improved}, scoring \cite{thibaudeau1993discrimination}, or Gibbs sampling \cite{meng1993maximum, larsen2001iterative} for parameter estimation have been developed to overcome the assumption of conditional independence in \cite{fellegi1969theory}. \cite{winkler2000frequency, winkler1988using} suggests a similar method that uses relative frequencies of the field values in place of dichotomous match statuses of individual fields to determine the weight parameters for linkage. The downside of the aforementioned family of models is that they disregard the generative patterns of each observation and thus suffer the loss of evidence contained in the actual values of the noisy observations.

In this regard, Bayesian methods have the advantage of naturally handling noises present in observations. By incorporating uncertainty into a generative process, Bayesian approaches allow the computation of match probabilities conditioned on the actual observations without neglecting the presence of simultaneous matches \cite{fortini2001bayesian}. Recent solutions to identity uncertainty have enjoyed the improvement of Bayesian inference techniques, and these solutions can largely be classified into ``parametric,'' ``cluster-based,'' or ``correlational'' depending on their generative assumptions and the criteria for inferring the co-reference structure.

The parametric family of methods is the most prominent, which encodes the latent linkage structure into a parameter of interest - such as matching matrices \cite{fortini2001bayesian, steorts2015entity, liseo2013some, tancredi2011hierarchical, steorts2014smered, steorts2015bayesian} or co-reference partitions \cite{sadinle2013generalized, sadinle2014detecting} - of a probabilistic generative model. Models in \cite{tancredi2011hierarchical} and \cite{liseo2011bayesian} represent the linkage structure as a one-hot matching matrix, and co-referent records separately undergo ``hit-miss'' distortions of latent ``true'' categorical or continous normal attributes. Similarly using matching matrices, \cite{steorts2014smered, steorts2015bayesian} provide a unified framework for record linkage and de-duplication that can be extended to simultaneously linking records across multiple files. \cite{sadinle2013generalized} developed a block-partitioning method to find co-referent partitions across multiple files under a normal mixture model, and \cite{jaro1989advances} dvelops on this model to schieve de-duplication. Apart from the generative models mentioned above, discriminative models such as 

Cluster-based methods bind similar objects into clusters, and these methods are more widely used in other branches of identity uncertainty such as author identification, text classification and noun coreference for which co-occurrence is well explained by close-knit groups of relevance. \cite{pasula2003identity} applies the concept of identity clustering to find the sample posterior over all relationships between objects, classes, and attributes modeled with the Relational Probability Model (RPM). \cite{bhattacharya2006latent} proposes an LDA-based model for the generation of author and citation entries in which authors and publications in the same membership group are more frequently observed together. Non-parametric Bayesian Dirichlet processes (DP) allow the number of such groups to be flexible, and \cite{dai2011grouped} applies DP to modeling groups of authors associated with topics. 

Correlational methods, on the other hand, compute the statistical interrelations of each pair of records rather than attempting to infer the linkage structure a posteriori. \cite{klami2013bayesian} uses covariance matrix of a correlational, multivariate normal generative model as a measure of statistical dependency used for finding the objects with the same identity. 

The aforementioned methods suffer from two critical downfalls when applied to cross-domain sequence linkage. First and foremost, these methods are tailored to the treatment of objects with unambiguous ground-truth features such as relational records with categorical, string-valued, or continuous attributes. As mentioned previously, this makes these methods fundamentally inapplicable to modeling event sequences. Also, the failure to isolate the unknown linkage structure when estimating hidden model parameters adds an extra layer of uncertainty which can compromise both the accuracy and convergence rate of the whole process. In addition, the mixture assumptions in many of the above methods are more ill-suited for sequence linkage than mixed-membership models since each entity exhibits a unique pattern of event generation in reailty while a mixture assumption binds them to restrictive patterns.

One notable method that particularly aims at sequence linkage is \cite{unnikrishnan2015asymptotically}, which computes the distance of two sequences in the simplex space of empirical distributions. Yet, as our case study in Section~\ref{sec:eval} reveals, this method fails when the empirical distributions have sparse intersections. In contrast, we propose a method that analyzes the semantics of each event incident, and determines the similarities of a pair of views within the latent semantic space, allowing for a more macroscopic pattern analysis. This method resolves all of the above-mentioned downfalls of existing methods through an information-theoretic interpretation of Latent Dirichlet Allocation as a mechanism for dimension reduction.

\subsection{PROBLEM FORMULATION}
\label{sec:prob_form}

Assume a world of $D$ real-world entities, $E=\{E_{d\in[D]}\}$, where $[D] = \{1,2,...,D\}$. $D_X$ and $D_Y$ each denotes the size of two data sets $X=\{X_{i\in[D_X]}\}$ and $Y=\{Y_{j\in[D_Y]}\}$. (We assume that $D=D_X$.) Entity $E_d$ generates exactly one sequence of events $X_d$ in the $X$ data set, and one or more sequences $Y_{j\in\pi(d)}$ in the $Y$ data set. Here, the function $\pi: [D_X]\rightarrow \mathcal{P}([D_Y])$ is the ``identity'' association, which indicates that sequence $X_i$ and sequences $Y_{j\in\pi(i)}$ are generated by the same real-world entity, namely $E_i$, and such sequences in different data sets are said to be `linked' or `co-referent.' ($\mathcal{P}([D_Y])$ is the set of all subsets of $[D_Y]$, and $\pi$ is a function whose image forms a partition of the set $[D_Y]$.) 

The sequence generation proceeds as follows. Each entity $E_{d\in[D]}$ performs a sequence of one of $W$ possible actions. (This categorical assumption can be easily extended to the continuous case.) When an action $w\in[W]$ is performed by etity $E_d$, this appends an entry $(w)$ in exactly one of the sequences among $X_d$ and $Y_{j\in\pi(d)}$. Every sequence is therefore an ordered collection of events, which is equivalent to the concept of a ``word'' in information theory and ``document'' in topic modeling. We will call this sequence, ``view.'' See Figure~\ref{fig:problem} for a visual understanding.

\begin{figure}[h!]
\centering
\makebox[\textwidth]{
\begin{tabular}{cc}
	\includegraphics[width=0.5\textwidth]{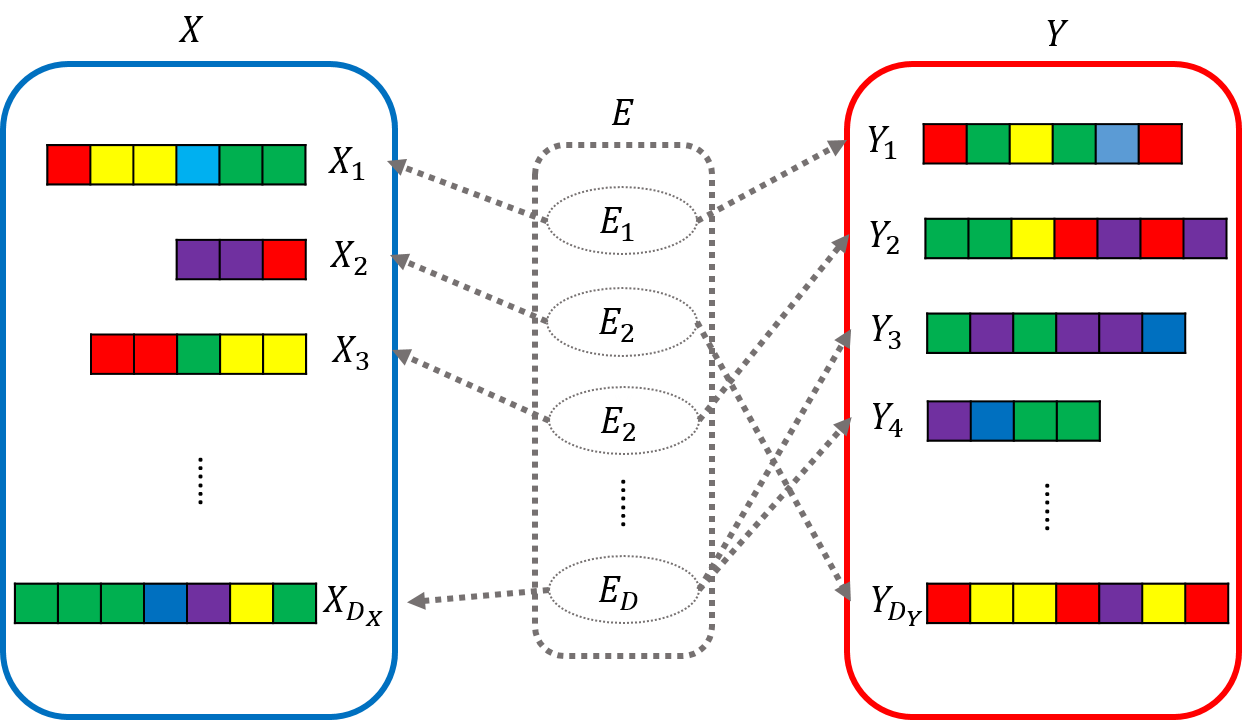} &
	\includegraphics[width=0.4\textwidth]{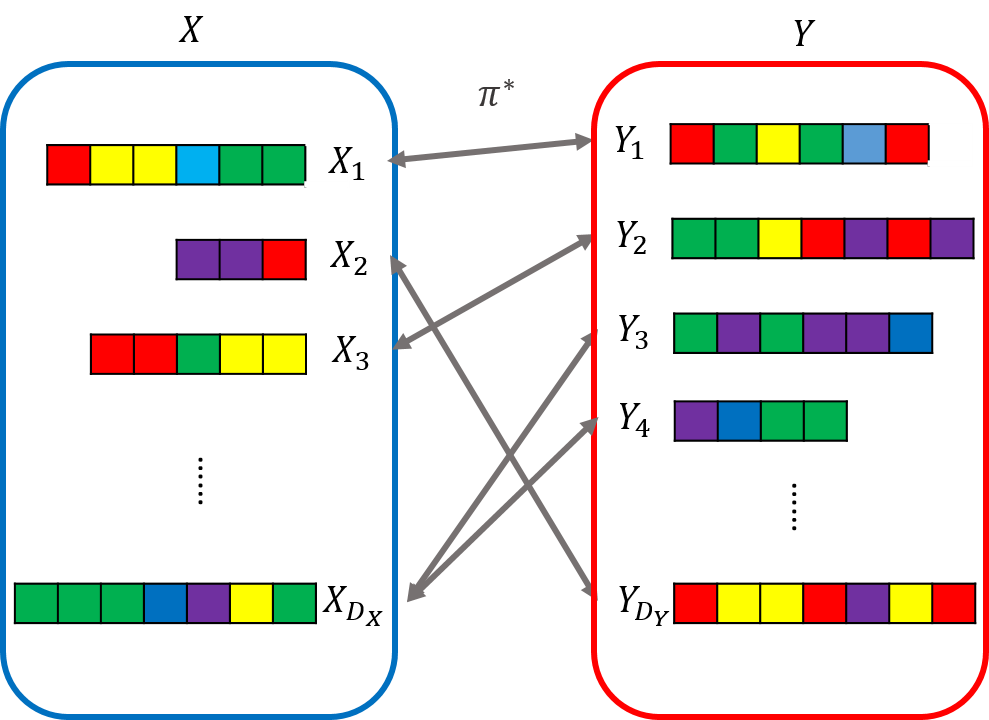} \\
\begin{tabular}{c}
(a) Generation of Views: \\ Hidden entities generate sequences across data sets.
\end{tabular}
&
\begin{tabular}{c}
(b) Sequence Linkage: \\ $\pi^*$ is the true linkage structure
\end{tabular}
\end{tabular}
}
\caption{Description of View Generation and the Linkage Problem. (Colors represent events.)}
\label{fig:problem}
\end{figure}

A typical problem of identity uncertainty is to find the exact identity mapping $\pi^*$. When the data sets $X$ and $Y$ consist of an identical number of entries, this becomes a well-known problem of finding the optimal bipartite graph matching between the $X$ and $Y$ views. When the data sets $X$ and $Y$ differ in size, however, graph matching methods are no longer straighforwardly applicable. In this paper, we study instead the problem of finding for each entity a set of up to $k$ candidates for a modest choice of $k$, which is a more realistic approach in cases where the number of co-referent target views are unknown or different for each entity. Although the candidate sets will no longer be mutually exclusive, this method has the additional advantage of preventing close misses at the slight cost of preciseness. We will call this type of matching ``Rank-$k$'' matching since it associates a view with $k$ supposedly most relevant candidate views.

\paragraph{Notations} 

In Sections~\ref{sec:algo} and the Appendix, we use $i$ and $j$ to index views in the $X$ and $Y$ domain data sets without a specific knowledge or reference to its real-world entity. To refer to a view in relation to a specific entity, we use the index $d$ so that $X_d$ and $Y_{\pi(d)}$ identify the $X$ and $Y$ views generated by entity $E_d$. Sometimes we will use $R$ to indicate an arbitrary view regardless of its domain. Although a view is a sequence of events, Split-Document model is a bag-of-words model in which events are exchangeable as we shall see in the later sections. We can thus represent a view as a vector of event frequencies and use the subsecripts $R_w$ (or equivalently, $X_{i,w}$, $Y_{j,w}$) to refer to the frequency of the event $w$ in view $R$ (or $X_i$, $Y_j$). Sometimes we may want to normalize the frequency vector to create a vector of relative frequencies (or an empirical distribution.) We denote this normalized relative frequency vector by $P$, and use superscripts such as $P^X$ and $P^Y$ to specify its domain. Lastly, we index topics with the index $k$, and each topic is a probability distribution over the events $w\in[W]$.

\section{LATENT DIRICHLET ALLOCATION AND VARIATIONAL BAYES}

\label{sec:lda_vb}

The proposed sequence linkage framework is closely related to Latent Dirichlet Allocation (LDA)\cite{blei2003latent}. Following subsections will review the LDA and the algorithms used for its statistical inference.

LDA is the simplest topic model, a Bayesian probabilistic model for generating documents. Each topic $\beta_k$ is a probability distribution over the vocabulary space. A total of $K$ topics are assumed, and each topic is drawn from a Dirichlet prior, $\beta_k\sim Dirichlet(\eta)$. Given these topics, each of the $D$ documents is generated in the following way. First a ``topic proportion'' $\theta_d\sim Dirichlet(\alpha)$ is drawn. Then for each word $w_{d,i\in[N]}$, a topic index $z_{d,i}$ is drawn according to the topic proportion, $z_{d,i}\sim\theta_d$, and $w_{d,i}\sim\beta_{z_{d,i}}$. Therefore the complete joint probability distribution becomes
\begin{equation}
	p(w,z,\theta,\beta) = p(\beta | \eta) \prod_d \left( p(\theta_d | \alpha) \prod_i \left( p(z_{d,i} | \theta_d) p(w_{d,i} | \beta_{z_{d,i}} ) \right) \right),
	\label{eq:joint}
\end{equation}
whose graphical model is shown in Figure~\ref{fig:lda}

\begin{figure}[h!]
\centering
\includegraphics[width=0.7\textwidth]{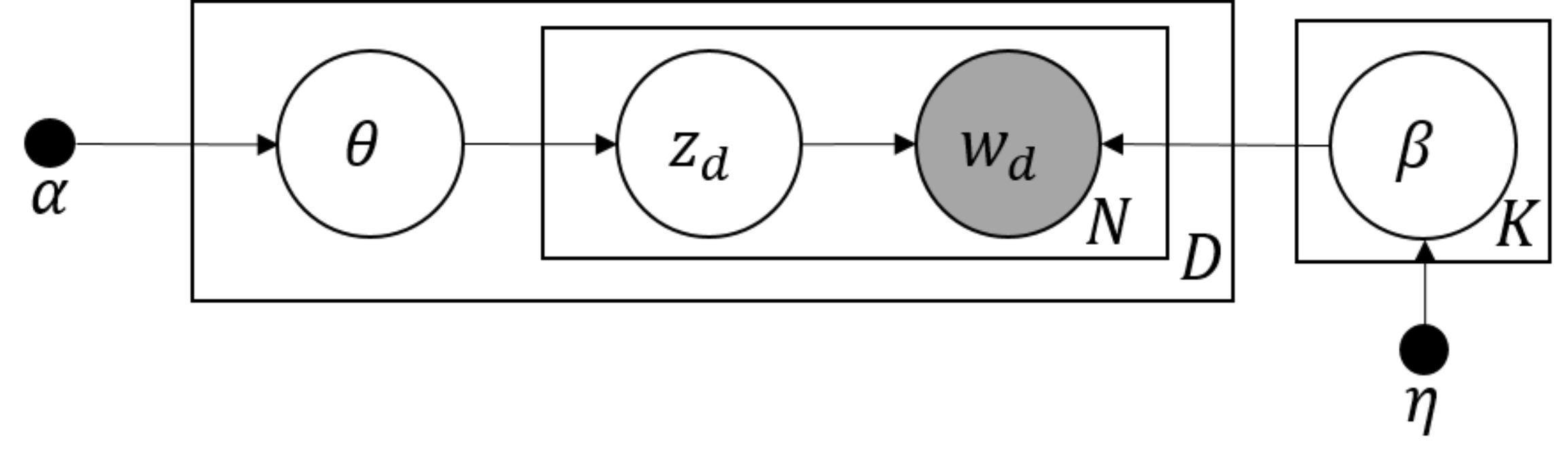}
\caption{Latent Dirichlet Allocation}
\label{fig:lda}
\end{figure}

Exact MAP inference on LDA is infeasible\cite{blei2003latent}, and usually one resorts to approximate inference techniques such as the MCMC and variational inference. In our approach, a particular variant of the variational inference is used, namely the stochastic variational inference\cite{hoffman2013stochastic}.

\subsection{VARIATIONAL BAYESIAN INFERENCE FOR LDA}

Variational Bayesian (VB) methods approximate an intractable posterior distribution to a member of a family of simpler distributions. In particular, it attempts to find the distribution $q^*(z,\theta,\beta)\in\mathcal{Q}$ that is closest in KL divergence to the true posterior distribution $p(z,\theta,\beta|w,\alpha,\eta)$,
\[
	q^* = \arg\min_{q \in \mathcal{Q}} KL\left(q(z,\theta,\beta) \| p(z,\theta,\beta| w)\right)
\]
where the conditioning hyperparameters in $p$ are omitted for simplicity. The simplest and most frequently utilized family $\mathcal{Q}$ is the ``mean-field'' family, which is a family of product distributions that is factorized into each latent variable term:
\begin{equation}
	q(z,\theta,\beta) = \prod_k q(\beta_k) \left( \prod_d q(\theta_d) \prod_i q(z_{d,i})\right)
	\label{eq:q}
\end{equation}
where $q(z_{d,i}=k)=\phi_{d,w_{d,i},k}$, $q(\theta_d) = Dir(\theta_d; \gamma_d)$, and $q(\beta_k) = Dir(\beta_k; \lambda_k)$ for conjugacy.

Instead of attempting to minimize the KL divergence directly, which involves the intractable distribution $p$, one may use the follwing relation
\begin{align*}
	KL(q \| p(z,\theta,\beta|w)) = - \ELBO(q) + \log p(w) = - \ELBO(q) + const
\end{align*}
where $\ELBO(q) = \Expof_q[\log q(z,\theta,\beta)] - \Expof_q[\log p(w, z,\theta,\beta)]$, and maximize $\ELBO(q)$ instead. $\ELBO(q)$ is called the Evidence Lower BOund (ELBO), which is conceptually the lower bound on $\log p(w)$ given by Jensen's Inequality. Early implementations of variational inference used the method of coordinate ascent that iteratevely maximizes the ELBO for each variational parameter while keeping others fixed \cite{wainwright2008graphical}. Although this method guarantees local convergence, it requires batch updates that become costly for large corpora, which led to the development of the online variational bayes that uses stochastic gradient descent for faster convergence \cite{hoffman2010online}.

\subsection{STOCHASTIC GRADIENT DESCENT AND ONLINE VB FOR LDA}

We first briefly discuss the nature of stochastic gradient descent before discussing online VB. Stochastic gradient descent (SGD) optimizes an object function when only the noisy estimates of the true gradient is available \cite{bottou2004stochastic}. Given an object function of the form $C(w) \triangleq \Expof{[Q(z,w)]}\triangleq{\int Q(z,w) dP(z)}$, SGD applies the following update formula 
\[
	w_{t+1} = w_t - \gamma_t H(z_t, w_t)
\]
to find the optimal value of $w$, where $z_t$ is an event from distribution $P(z)$ and $\gamma_t$ is the learning rate. The update term $H(z,w)$ satisfies the condition 
\[
	\Expof_z{[H(z,w)]} = \nabla_w C(w),
\]
and thus can be understood as a noisy yet consistent estimate of the true gradient of $C$ at $w$. It is shown that $w_t$ converges almost surely to the local optimum \cite{bottou1998online}.

\paragraph{Stochastic Maximization of the ELBO}

SGD can be used to optimize the ELBO for LDA \cite{hoffman2013stochastic}\footnote{SGD can be applied in VB for any generative model that involves local and global latent variables \cite{hoffman2013stochastic}.}, $\ELBO(Q) \triangleq \Expof{\left[\log \frac{q(z,\theta,\beta)}{p(w,z,\theta,\beta}\right]}$. Considering the factorization of $p$ and $q$ from Equations~\ref{eq:joint} and~\ref{eq:q},
\[
	\ELBO(Q) = \Expof_q{\left[\log \frac{p(\beta|\eta)}{q(\beta)} \right]} + 
	\sum_d \left( 
		\Expof_q{\left[ \log \frac{p(\theta_d | \alpha)}{q(\theta_d)} \right]}  + 
		\Expof_q{\left[ \log \frac{p(w_d, z_d | \theta_d, \beta)}{q(z_d)} \right]}
	\right).
\]

Letting $I\sim \mbox{Unif}(1,...,D)$ be a random variable that chooses an index $i$ over the document indices $[D]$, we can rewrite $\ELBO(Q)$ as $\Expof_I\left[\ELBO_I(Q)\right]$ where
\[
	\ELBO_I(Q) = 
		D \Expof_q{\left[\log \frac{p(\beta|\eta)}{q(\beta)} \right]} + 
		\Expof_q{\left[ \log \frac{p(\theta_d | \alpha)}{q(\theta_d)} \right]}  + 
		\Expof_q{\left[ \log \frac{p(w_d, z_d | \theta_d, \beta)}{q(z_d)} \right]},
\]
so that the ``natural gradient'' of $\ELBO_I$ with respect to each global variational parameter $\eta$ is a noisy yet unbiased estimate of the natural gradient of the variational objective, $\ELBO$. Computing the natural gradient instead of the Euclidean gradient corrects for the geometry of the space of the variational parameters by using the symmetrized KL divergence as the measure of spatial distance \cite{amari1998natural} and thus leads to a more effective convergence to the local optimum \cite{hoffman2013stochastic}. Refer to \cite{hoffman2010online} for the resulting online variational Bayesian algorithm for LDA.

\section{SPLIT-DOCUMENT MODEL AND LDA-Link ALGORITHM}
\label{sec:algo}

We now introduce \emph{Split-Document} model, a simple probabilistic generative model for co-referent views that is based on the LDA. Based on this model we suggest \emph{LDA-Link} as a solution for identifying co-referent views across data sets of different domains.

\subsection{SPLIT-DOCUMENT MODEL}

\begin{figure}[h!]
\centering
\includegraphics[width=0.7\textwidth]{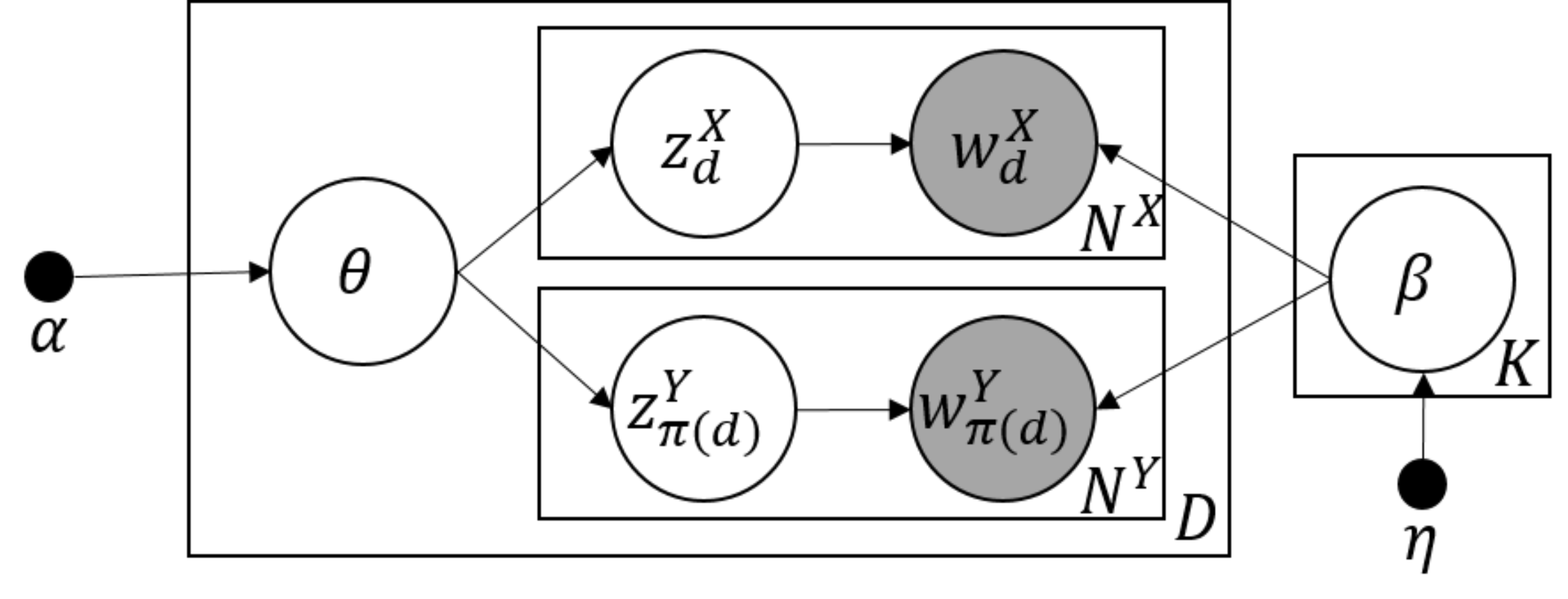}
\caption{Split-Document Model}
\label{fig:split_doc}
\end{figure}

The Split-Document model displayed in Figure~\ref{fig:split_doc} extends LDA to model the generation of views across domains. In this model, a real-world entity generates a sequence of i.i.d. events according to LDA, and each event falls into exactly one among multiple co-referent views. (For now we assume that each entity generates a single view in each data set, and that views in the smae data set have a fixed number of events.) A set of co-referent views are thus generated through a mixture model whose mixture proportions depend on the sequence-generating entity.

Although the Split-Document model is kept simple in this paper for the sake of lucid presentation, its assumptions can be generalized to allow more than two data sets with each entity generating more than one views in each data set and events ocurring in different data sets with different probabilities. As in \cite{steorts2014smered, steorts2015bayesian}, this approach would combine the problem of record linkage with deduplication.

\subsection{LDA-Link ALGORITHM}

We now introduce LDA-Link, a co-reference linkage algorithm based on the Split-Document model. The key idea is to consider topic proportions as a reduction of dimensionality from the size of the entire event space $W$ to the number of topics $K$, and to compare these dimension-reduced representations. This distinguishes LDA-Link from other methods that leverage only the common events that appear in both views, which causes them to fail when a sequence pair displays sparse common occurrence.

The algorithm works in three separate phases. First in the ``Learning'' phase, topics are estimated from the views in the two domains. In the second phase, the topic estimates are used to find the topic proportions $\theta$'s for each view that maximizes the posterior distribution given the estimated topics. This phase is the ``Dimensionality Reduction'' phase that effectively reduces the dimension of each view from $L-1$ to $K-1$. A score is then computed for every pair of views $(X_i, Y_j)$ as the Jensen-Shannon distance between their topic proportions. In the last ``Rank-$k$ Linkage'' phase, up to $k$ candidate $Y$ views are declared as potential matches for each $X$ view based on these scores.

The rest of this section discusses each phase in detail. Appendices~\ref{sec:topic_conv} through~\ref{sec:proofs} will explore the guarantees of this algorithm in theory.

\paragraph{Learning Phase: Topic Estimation}

In the Split-Document model, co-referent views share the same topic proportion, and these views are essentially a bipartition of a document generated through the normal, un-split LDA process described in Figure~\ref{fig:lda}. If the true co-reference linkage structure is known, estimating the topics for Split-Document model would amount to finding the MAP LDA topics where all co-referent views are combined into one document. Yet, since the co-reference structure is the unknown that we aim to find, it is difficult to directly compute the MAP estimates of the generative model in Figure~\ref{fig:split_doc}.

A workaround is to consider a slightly different surrogate model in which each view is considered as a separate document that has a topic proportion of its own right, and instead learn the topics optimal to this model as an approximation. This surrogate model is called the \emph{Independent-View} model shown in Figure~\ref{fig:indep_view}. Intuitively, in large-document limits where the number of words reaches infinity, learning the Independent-View model is equivalent to learning LDA with duplicates of each document. (We study the effectiveness of this surrogate learning in Appendix~\ref{sec:topic_conv})

\begin{figure}[h!]
\centering
\includegraphics[width=0.7\textwidth]{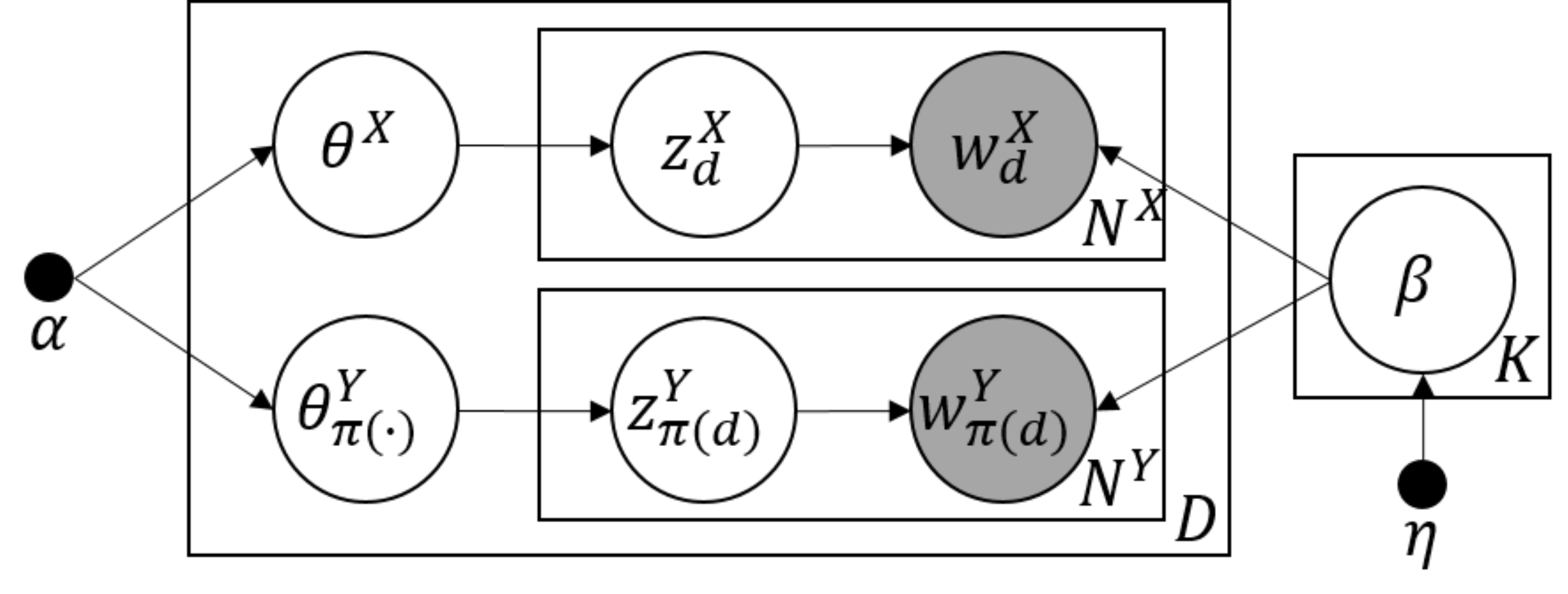}
\caption{Independent-View Model (Surrogate for Learning Topics in Split-Document Model)}
\label{fig:indep_view}
\end{figure}

We will call the ELBO of the Independent-View model $\ELBO'$, which is given by
\begin{equation}
	\ELBO' = \sum_d l'_d = \sum_d \left( l\left(X_d, \phi^X_{d}, \gamma^X_{d}, \lambda\right) + l\left(Y_{\pi(d)}, \phi^Y_{\pi(d)}, \gamma^Y_{\pi(d)}, \lambda\right) \right)
\label{eq:ELBO'}
\end{equation}
where $l$ is given as
\begin{equation}
	l(R,\phi,\gamma,\lambda) = \sum_w R_{w} \sum_k \phi_{w,k} \left(\Expof_{\gamma}[\log\theta_{k}] + \Expof_{\lambda}[\log\beta_{k,w}] - \log\phi_{w,k}\right) + \frac{1}{2D}f(\lambda),
\end{equation}
and $f(\lambda) = \sum_k -\log\left(\sum_w \lambda_{k,w} + \sum_w(\eta-\lambda_{k,w})\Expof_{\lambda}[\log\beta_{k,w}] + \log\Gamma(\lambda_{k,w})\right)$. Note its difference from the ELBO of the Split-Document model, which is
\[
	\ELBO = \sum_d l_d = \sum_d 2l\left(X_d + Y_{\pi(d)}, \phi_d, \gamma_d, \lambda\right).
\] 

Algorithm~\ref{alg:LDA} summarizes the Learning phase.

\begin{algorithm}[t]
	\caption{Learning Phase: Topic Estimation}
	\label{alg:LDA}
	\begin{algorithmic}
	\Require $P^X_{d,w} = \frac{1}{N}X_{d,w}, P^Y_{d,w} = \frac{1}{N}Y_{d,w}$\\
	Define $\rho_t \triangleq (\rho_0 + t)^{-K}$. Initialize $\lambda \in \mathbb{R}$
		\For{$P \in \cup_{d\in[D]}\{P^X_d, P^Y_d\}$}
			\State Initialize $\gamma_{d,k} = 1$
			\Repeat
				\State $\phi_{w,k} = \frac{\exp\{\Expof_q[\log \theta_{k}] + \Expof_q[\log \beta_{k,w}]\}}{\sum_k\exp\{\Expof_q[\log \theta_{k}] + \Expof_q[\log \beta_{k,w}]	\}}$
				\State $\gamma_{k} = \alpha + N \sum_w \phi^X_{w,k} P_{w}$
			\Until{$\gamma_d$ has converged}
			\State Compute $\tilde{\lambda}_{k,w} = \eta + DNP_{d,w}\phi_{w,k}$
			\State $\lambda = (1 - \rho_t) \lambda + \rho_t \tilde{\lambda}$
		\EndFor\\
		\Return $\lambda$
	\end{algorithmic}
\end{algorithm}

\paragraph{Dimension Reudction and Rank-$k$ Linkage Phases}

With the topics obtained in the learning phase, the topic proportions are learned for each view using the conventional coordinate descent method, mapping each view to a vector on the $K-1$-dimensional latent semantic simplex space. Once these topic proportions are learned, a dissimilarity score based on Jensen-Shannon distance is computed for every pair of views $(X_i, Y_j)$ as
\begin{equation}
	score(X_i, Y_j) = JS(\theta^X_i, \theta^Y_j) = KL\left(\theta^X_i\left\|\frac{1}{2}\left(\theta^X_i + \theta^Y_j\right)\right.\right) + KL\left(\theta^Y_j\left\|\frac{1}{2}\left(\theta^X_i + \theta^Y_j\right)\right.\right).
\end{equation}

Algorithm~\ref{alg:score} summarizes the Dimension Reduction phase.

\begin{algorithm}[h!]
	\caption{Dimensino Reduction Phase}
	\label{alg:score}
	\begin{algorithmic}
	\Require $P^X_{d,w} = \frac{1}{N}X_{d,w}, P^Y_{d,w} = \frac{1}{N}Y_{d,w}$ for $d\in[U],w\in[W]$\\
	Initialize $\gamma^X_d = \vec{\mathbf{}{1}}$, $\gamma^Y_d = \vec{\mathbf{1}}$
		\For{$\DOM \in \{X,Y\}$}
			\Repeat
				\State $\phi^\DOM_{w,k} = \frac{\exp\{\Expof_q[\log \theta^\DOM_{k}] + \Expof_q[\log \beta_{k,w}]\}}{\sum_k\exp\{\Expof_q[\log \theta^\DOM_{k}] + \Expof_q[\log \beta_{k,w}]	\}}$
				\State $\gamma^\DOM_{k} = \alpha + N \sum_w \phi^\DOM_{w,k} P_{w}$
			\Until{$\gamma^\DOM_d$ has converged}
		\EndFor \\

		\Return $\gamma^X$ and $\gamma^Y$
	\end{algorithmic}
\end{algorithm}

The final Rank-$k$ Linkage phase selects for each $X$ view, $k$ candidate views of the smallest dissimilarity score, as shown in Algorithm~\ref{alg:ranking}.

\begin{algorithm}[h!]
	\caption{Rank-$k$ Linkage Phase}
	\label{alg:ranking}
	\begin{algorithmic}
	\Require $S = [0,\infty]^{U\times U}$, where $S_{i,j} = score(X_i, Y_j)$
		\For{$i \in [U]$}
			\State $S \in [0,\infty]^U$, where $S_j = score(X_i, Y_j)$
			\State $\pi^*(i) = \{j_{1,...,k} | j_{1,...,U} \mbox{ is a permutation of } [U] \mbox{ such that } S_{j_a} < S_{j_b} \mbox{ if } a<b\}$
		\EndFor\\
		\Return $\pi^*$
	\end{algorithmic}
\end{algorithm}

\section{CASE STUDY: BREAKING ANONYMITY IN LOCATION-BASED SOCIAL MEDIA}
\label{sec:eval}

We now apply LDA-Link to location-based social media profile linkage, where profiles of the same individual in different social media are matched based on their online activities. As explained in the introduction, location-based data sets present critical challenges of sparse common occurrence. The ability to reconcile identities and bridge data across online social platforms of different thematic nature has significant commercial as well as privacy implications \cite{riederer2016linking}. First the data sets and the baseline algorithms are explained, and the performance of LDA-Link is evaluated.

\subsection{DATA SETS}

A total of two pairs of datasets were used, both of which contain only the spatio-temporal information of public activities of profiles in two different social media collected during a common time span. Both of these datasets were studied and explained previously in \cite{riederer2016linking}.

\begin{itemize}

	\item \textbf{Foursquare-Twitter (FT): }
This dataset contains checkins on Foursquare and posts on Twitter, both of which are  geo-temporally tagged. Each account activity is therefore a (user id, time, GPS coordinate) tuple. By selecting only the users who has records in both social media accounts, a total of 862 users, 13,177 Foursquare checkins, and 174,618 tweets were obtained. The imbalance of activities in the two social media is a factor of challenge.

While Foursquare checkins are typically associated with a user exposing their current activities, tweets are associated with more general behaviors. In this dataset, only 260 pairs of checkins (less than 0.3\%) had exactly matching GPS coordinates, and none of them were made within 10 seconds of each other, suggesting that it is highly unlikely that there is a pair of account activities forwarded by software across both services \cite{riederer2016linking}.

	\item \textbf{Instagram-Twitter (IG): }
The second dataset is also a collection of (user id, time, GPS coordinate) from public posts on the photo-sharing site Instagram, and the microblogging service Twitter. This data set was obtained by linking Instagram and Twitter accounts that were associated with the same URL in their user profiles, and downloading the spatio-temporal tags of the tweets made by these Twitter accounts \cite{riederer2016linking}. The collection includes 1717 unique users, 337,934 Instagram posts, and 447,336 Tweets.

\end{itemize}

Both pairs of data sets contain the timestamped visits of each user, and while users are communicating an action or a message to the general public, the events(posts) in each data set are collected within different thematic contexts. Split-Document model dictates that each recorded visit occurs with a certain purpose, e.g., shopping, sports, travel, hobbies etc.. LDA-Link attempts to discover such motifs (or ``topics'') and associate with every view a proportional mixture of these topics to represent its characteristic feature.

\paragraph{Modulating the Sparsity of Common Events with Spatiotemporal Granularity}
The full GPS coordinates and timestamps of the posts in each of these data sets never coincide precisely, which raises the problem of determining the meaning of ``common occurrence.'' Our solution is to bin time and locations based on spatiotemporal proximity, and declare events belonging to the same bin to have occurred in common. The size of the bin can be controlled to represent different levels of measurement precision, and changing the event space granularity in this manner modulates the discreteness and size of the event space. This allows the study of the robustness of the algorithm under different degrees of sparsity available to the linking agent. In our experiment we bin event locations by truncating the coordinate values after a certain number of digits below decimal and bin event timestamps into a fixed number of bins. We call each of these numbers ``spatial'' and ``temporal'' granularity.

\subsection{PRIOR ALGORITHMS}

Here we summarize three baseline algorithms for identity reconciliation that were inspired by state-of-the-art algorithms in the social computing literature.

\paragraph{Sparsity-Based Scoring: The ``Netflix Attack'' (NFLX)}

Based on the algorithm used to de-anonymize the Netflix prize dataset in \cite{narayanan2008robust}, \cite{riederer2016linking} describes a variation for cross-domain reconciliation, where a score between views $X_i$ and $Y_j$ is defined as
\[
	S(X_i, Y_j) = \sum_{(l,t)\in X_i \cap Y_j} w_l f_l(X_i, Y_j),
\]
where
\[
	w_l = \frac{1}{\ln\left(\sum_{j}Y_j(l) \right)} \mbox{, and } f_l(X_i, Y_j) = \exp\left(\frac{X_i(l)}{n_0}\right) + \exp\left(-\frac{1}{X_i(l)}\sum_{t:(l,t)\in X_i}\min_{t':(l,t)\in Y_j} \frac{|t-t'|}{\tau_0}\right),
\]
and $n_0, \tau_0$ are model parameters.

This algorithm relies on the exact timestamps. The algorithm matches an $X$ view with the $Y$ view with the smallest score, and leaves it unmatched if the best candidate and the second best differ in scores by no more than a $\epsilon$ standard deviations. In resemblance to \cite{narayanan2008robust} this score favors locations that are rarely visited, frequent visits to the same location, and visits that occur close in time, thus exploiting sparsity. \cite{riederer2016linking}

\paragraph{Density-Based Scoring: JS-Distance Matching (JS-Dist)}

In \cite{unnikrishnan2015asymptotically}, authors proved the asymptotic optimality of the JS-Distance between relative frequencies (empirical distributions) of two observation sequences as a measure of disparity.
\[
	S(X_i, Y_j) = JS(X_i, Y_j) = KL\left(X_i \left\| \frac{1}{2}\left(X_i + Y_j\right)\right.\right) + KL\left(Y_j \left\| \frac{1}{2}\left(X_i + Y_j\right)\right.\right)
\]

Their algorithm estimates the true matching as the matching that mininimizes the sum of the JS measures. It relies on the density of the data based on the asymptotic convergence of empirical distributions implied by Sanov's Theorem.

\paragraph{Leveraging Both Sparsity and Density: Poisson Process (POIS)}

\cite{riederer2016linking} suggests a simple generative model for mobility records in which the number of visits to each location within a certain time period follows a Poisson distribution whose rate parameters are specific to the location and period of the visit. Based on this model, the following similarity score between two views can be defined for an MAP estimate of the identity linkage structure:
\[
	S(X_i, Y_j) = \sum_{l,t} \ln\phi_{l,t}(X_i(l,t), Y_j(l,t)),
\]
where $l$ and $t$ are location and time indices and
\[
	\phi_{l,t}(x,y) = \frac{e^{-\lambda p_1 p_2} (1-p_1)^y (1-p_2)^x}{(\lambda(1-p_1)(1-p_2))^{\min(x,y)}}\Expof\left[\frac{(X+\max(x,y))!}{(X+|x-y|)!}\right].
\]

The identity mapping is the mapping that maximizes the expected sum of scores.

\subsection{PERFORMANCE ANALYSIS}

We now present the empirical performances of LDA-Link. In light of the ``Rank-$k$'' matching we described in Section~\ref{sec:prob_form}, we measure our performance in terms of the Rank-$k$ recall, which is the proportion of views in the source data set whose co-referent views in the source data set are fully contained in the set of $k$ best candidates, not allowing ties. We draw our attention to LDA-Link's relative performance as compared to the baseline algorithms, and evaluate its robustness against sparse common occurrence by (1) modulating the granularity of the event space, and (2) testing on a sample of sequences with sparse common events.

\paragraph{LDA-Link and Domain-Specific Alternatives}

Figure~\ref{fig:recall} plots the best Rank-$k$ recalls of each algorithm for different values of $k$. LDA-Link outperforms the domain-specific reconciliation algorithms as the size of the candidate set $k$ is increased. Although NFLX and POIS perform better for small values of $k$, LDA-Link's recall increases more rapidly, exceeding POIS and NFLX respectively at $k=10$ (1.16\% of the total number of views) and $k=19$ (2.20\%) for FSQ-TWT, and at $k=18$  (1.05\%) and $k=23$ (1.34\%) for IG-TWT. The early plateu for POIS occurs due to the lack of rules for evaluating the similarity of a pair of views when none of their events belong to the same location or time bins. The plateau is reached more slowly at a higher recall for NFLX because NFLX depends on precise time-difference instead of binning by time and is thus slightly less vulnerable to time granuarity. LDA-Link, on the other hand, computes the similarity of a pair of views on a dimension-reduced space of topic proportions, which removes the dependence on the granularity of the event space. When $k$ reaches up to 10\% of the total number of views, LDA-Link outperforms POIS and NFLX by over 50\% and 20\% on FSQ-TWT. 

\newpage

\begin{figure}[t]
\centering
\makebox[\linewidth]{
\begin{tabular}{c}
\includegraphics[width=\textwidth]{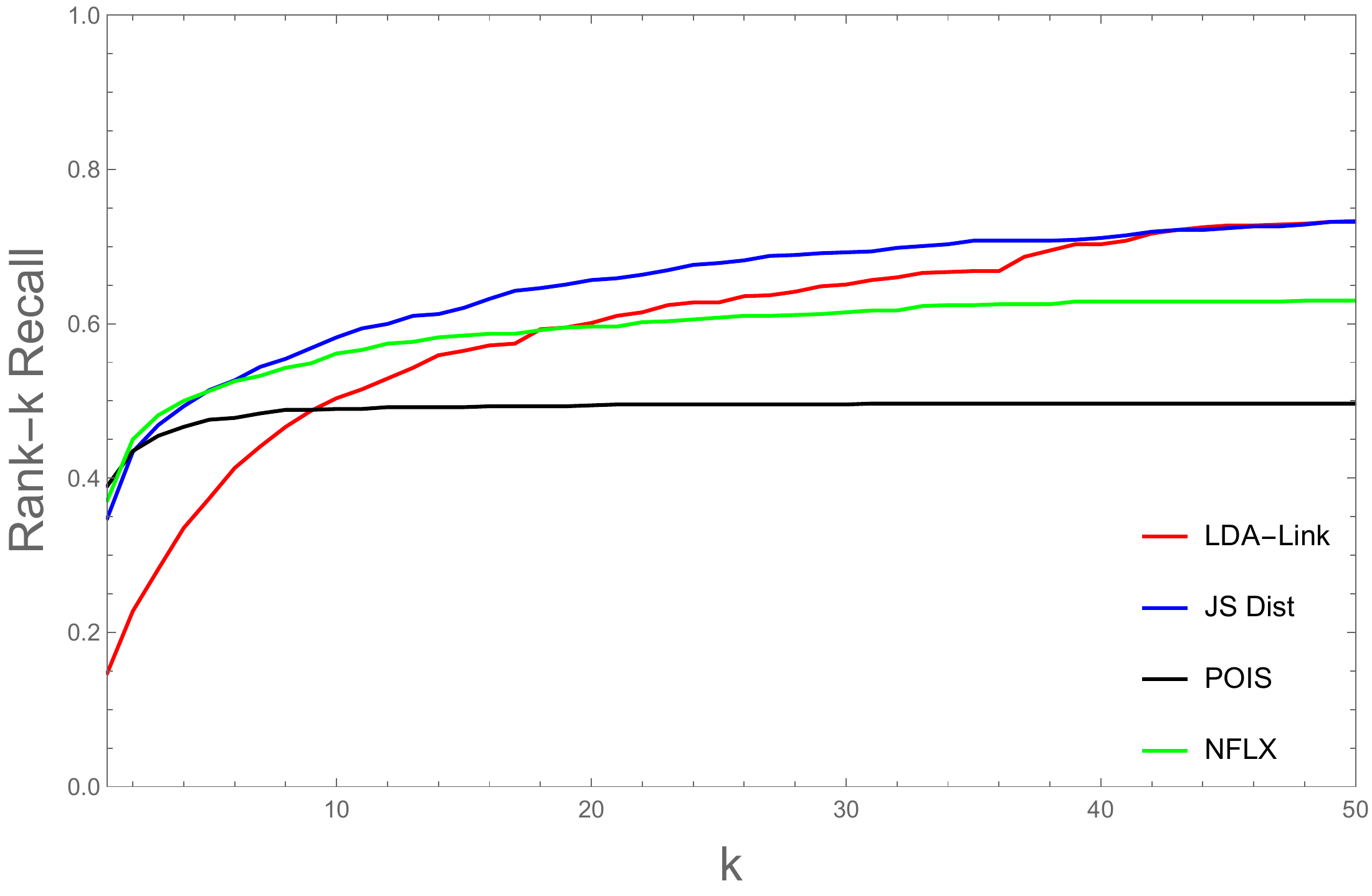} \\
(a) Recall on FSQ-TWT \\
\includegraphics[width=\textwidth]{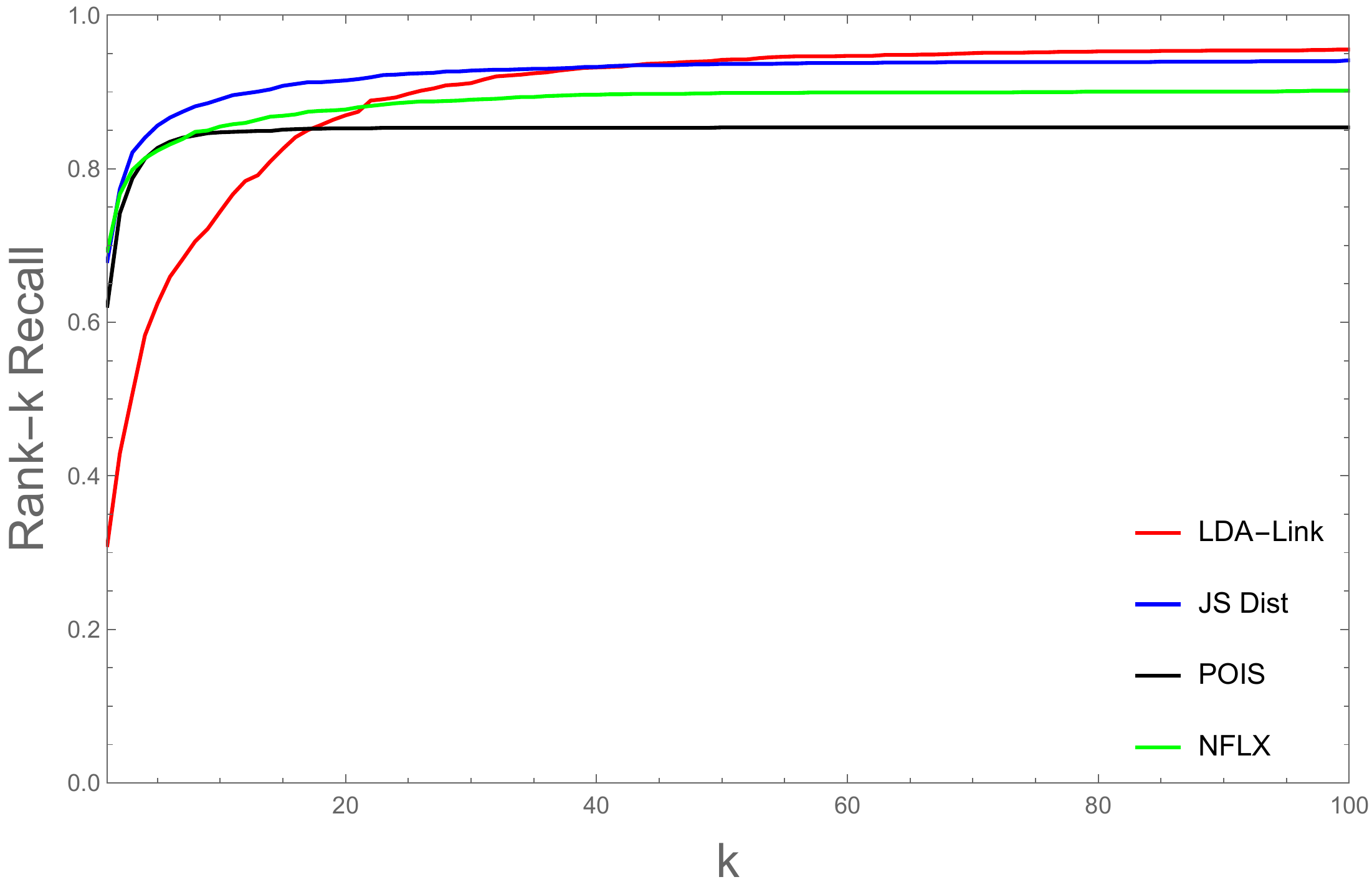} \\
(b) Recall on IG-TWT \\
\end{tabular}
}
\caption{Best Rank-$k$ Recall Plots on the Two Datasets}
\label{fig:recall}
\end{figure}

\newpage

\begin{landscape}

\begin{figure}[t]
\centering
\makebox[\linewidth]{
\begin{tabular}{cc}
\includegraphics[width=0.8\textwidth]{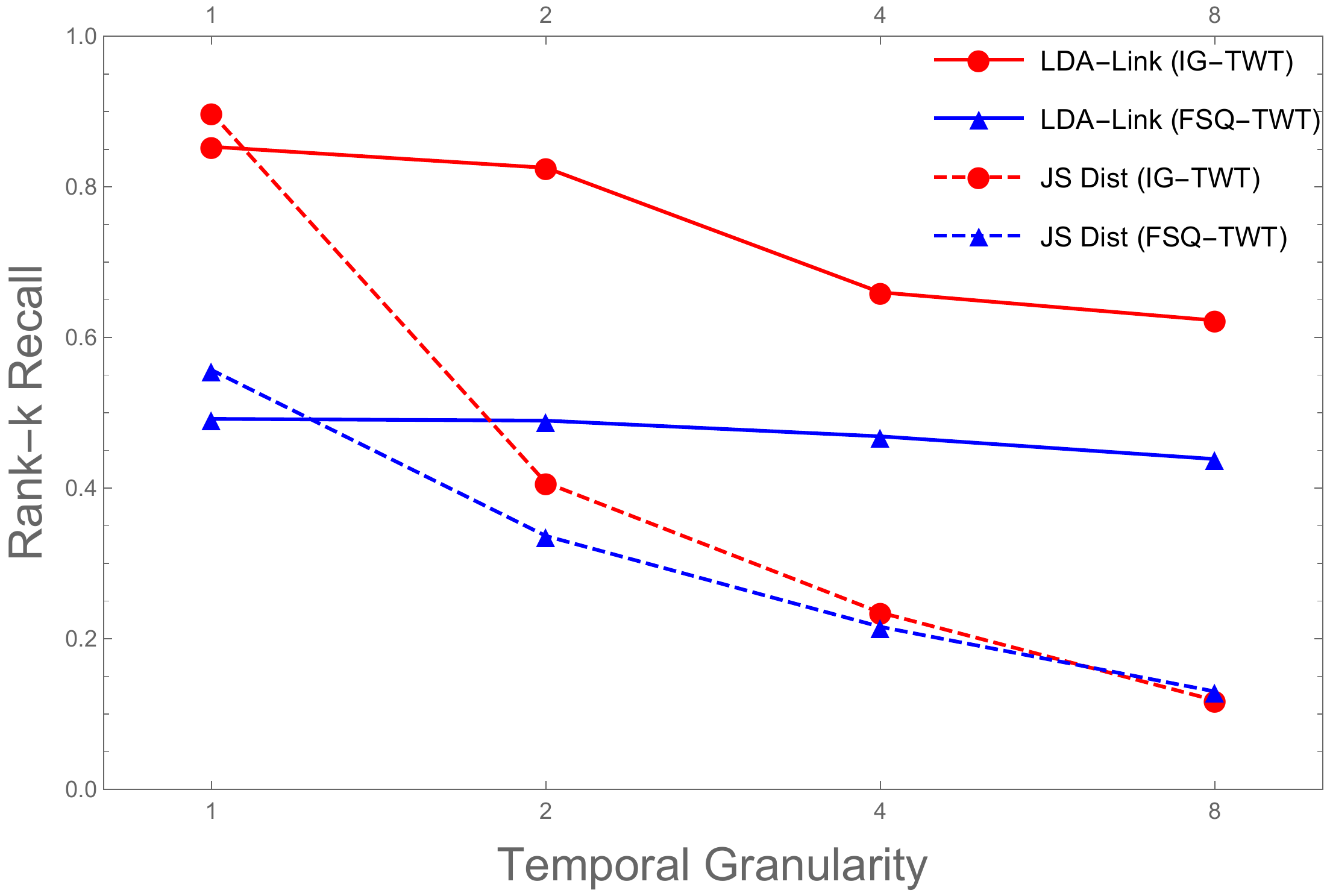} &
\includegraphics[width=0.8\textwidth]{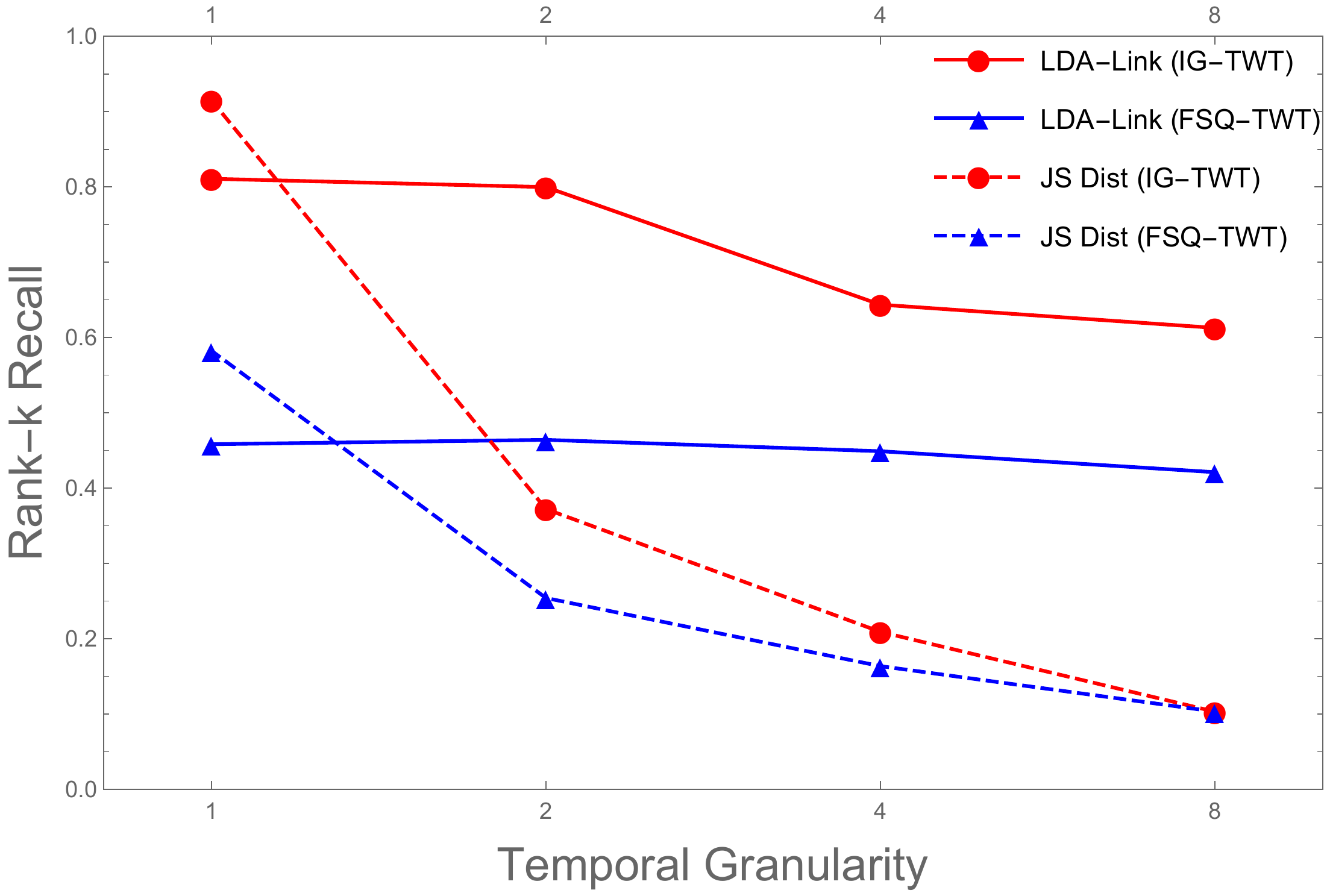} \\
(a) Spatial Granularity = 0&
(b) Spatial Granularity = 1\\
\includegraphics[width=0.8\textwidth]{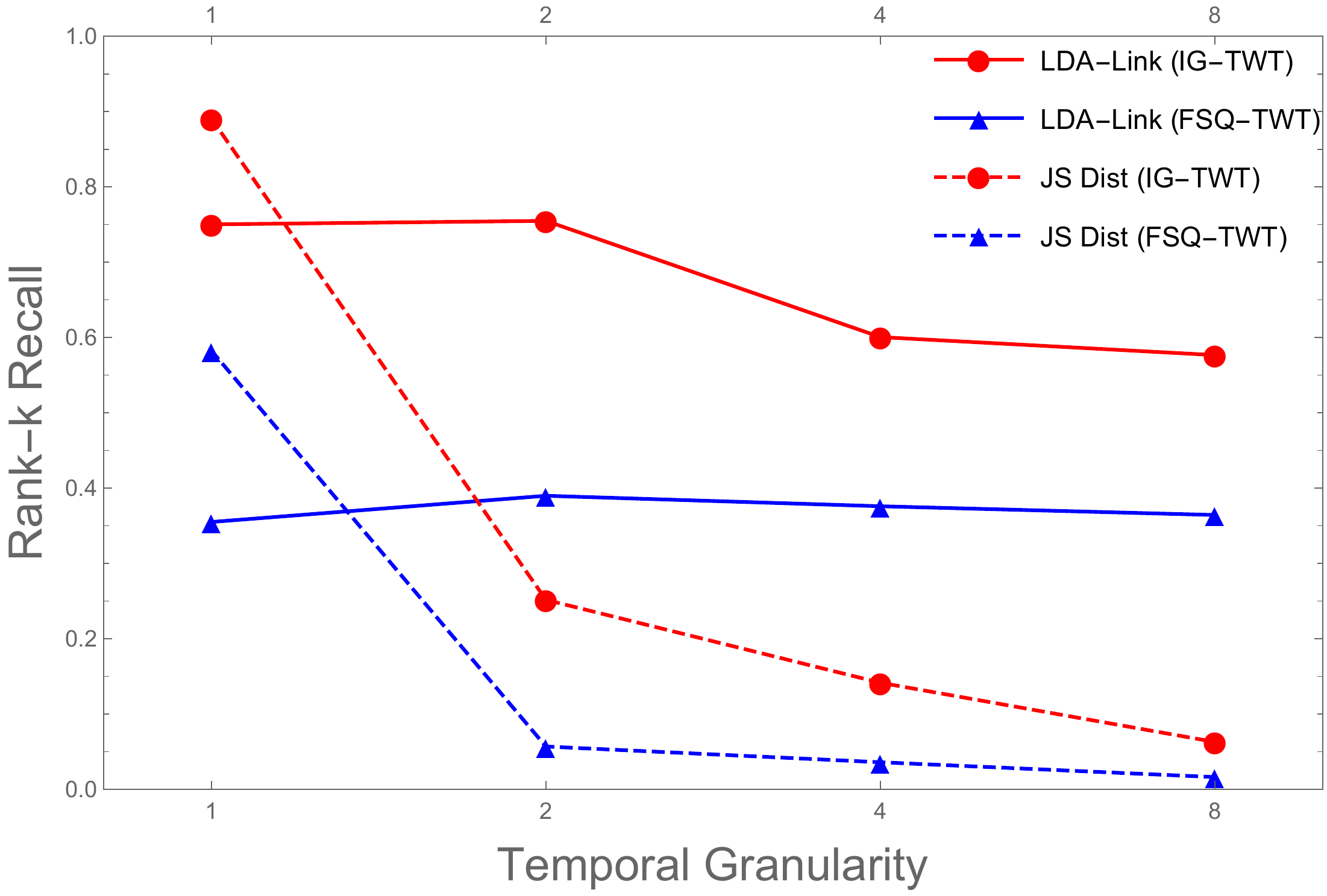} &
\includegraphics[width=0.8\textwidth]{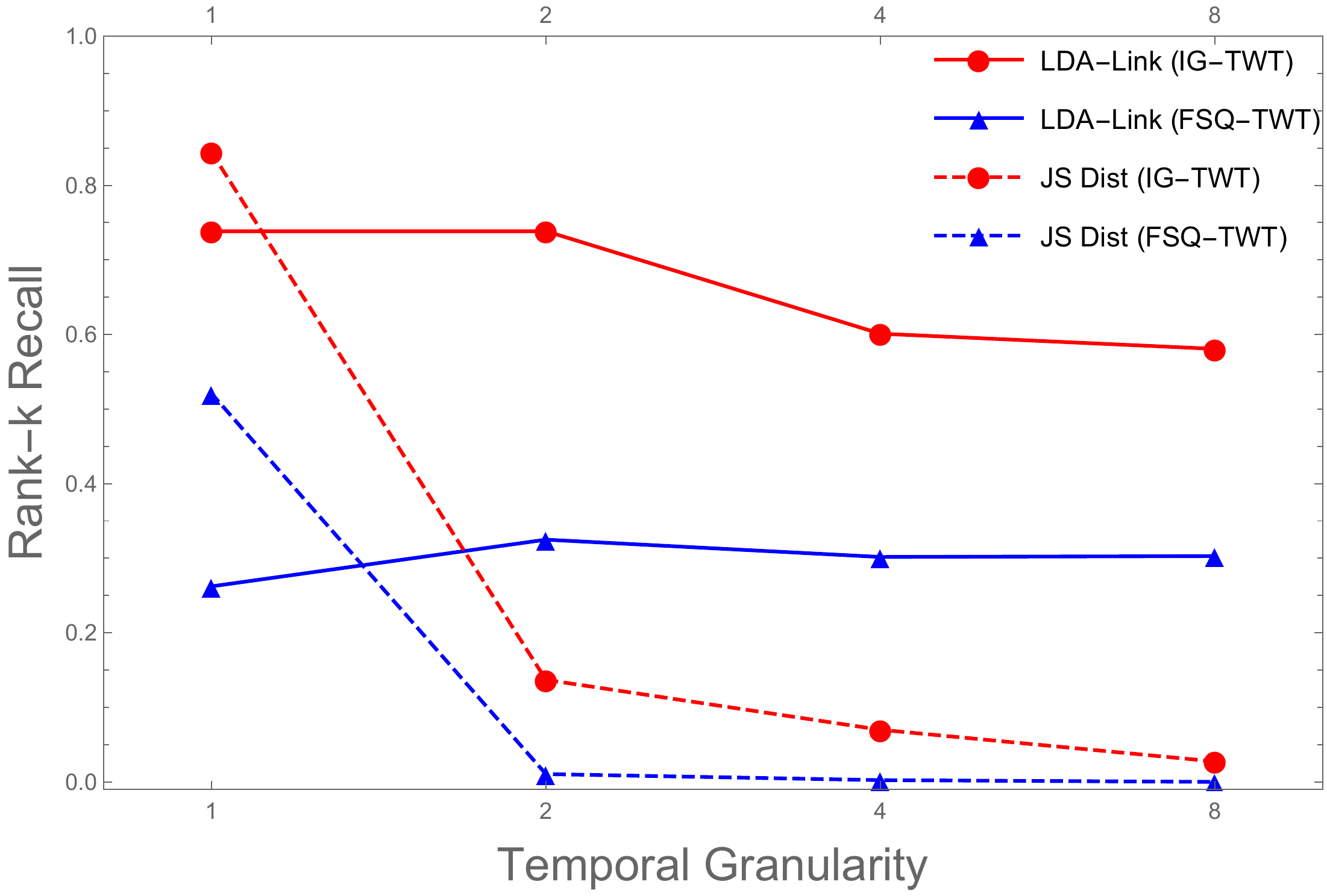} \\
(c) Spatial Granularity = 2&
(d) Spatial Granularity = 3\\
\end{tabular}
}
\caption{Rank-$k$ Recall of LDA-Link and JS-Dist for Different Spatiotemporal Granularities}
\label{fig:gran}
\end{figure}

\end{landscape}

\newpage

\begin{figure}[t]
\centering
\makebox[\textwidth]{
\begin{tabular}{c}
\includegraphics[width=\textwidth]{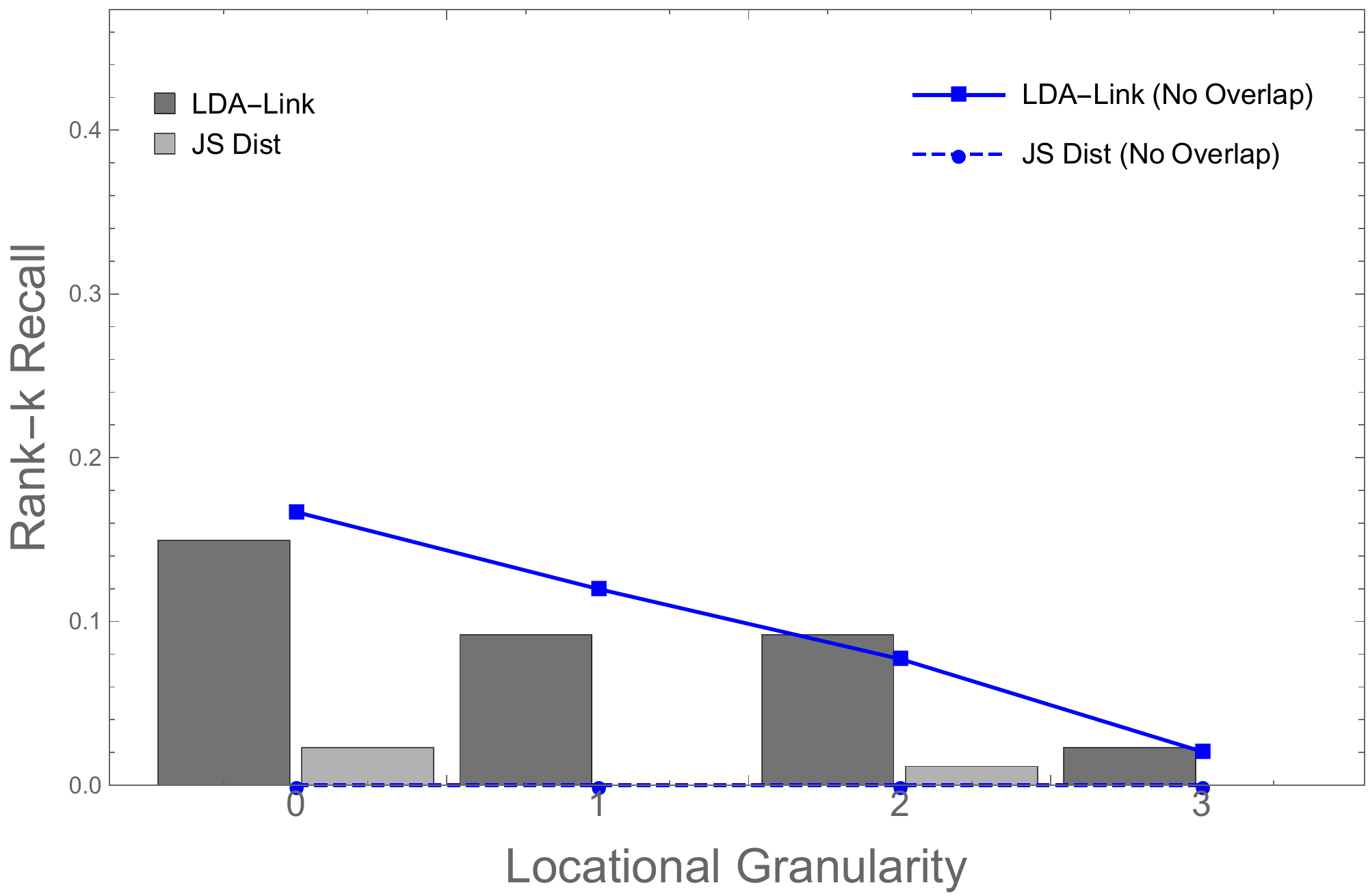} \\
(a) FSQ-TWT ($k$=10)\\
\includegraphics[width=\textwidth]{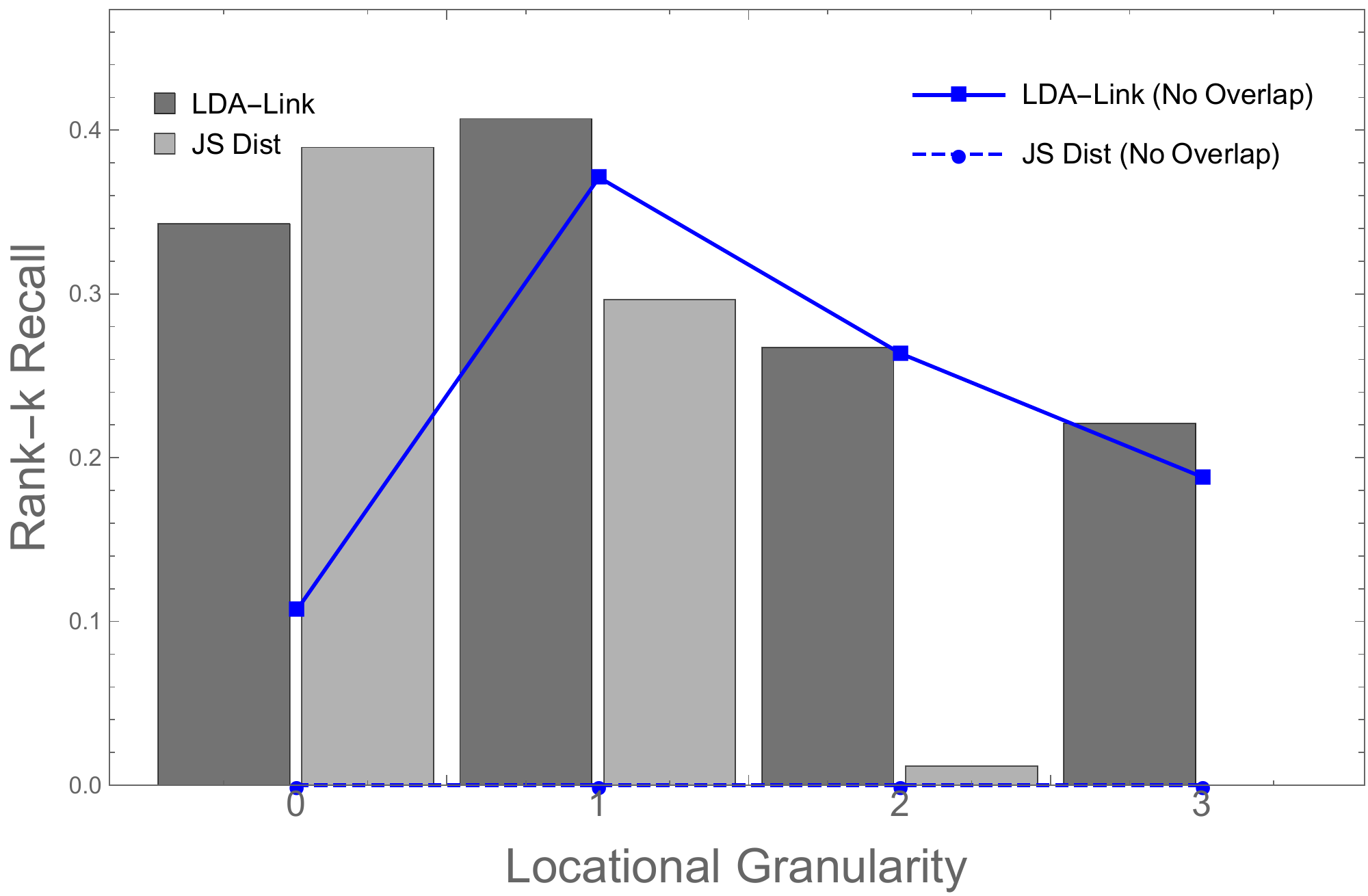} \\
(b) IG-TWT ($k$=20)
\end{tabular}
}
\caption{Rank-$k$ Linkage Recall of Sequence Pairs with Top 10\% L1 Distance}
\label{fig:overlap}
\end{figure}

\newpage

\paragraph{Robustness against the Event Space Granularity}

In Figure~\ref{fig:gran}, we tested LDA-Link against different levels of spatiotemporal granularity of the event space. The number of learned topics were fixed to $K=500$ and $K=600$ for FSQ-TWT and IG-TWT respectively. The plot displays the Rank-$k$ matching of LDA-Link and JS-Distance matching (dashed lines) on the two data sets, where $k$ was set to 10 and 20 for FSQ-TWT and IG-TWT respectively for a consistent comparison. Although the rank-$k$ recall for JS-Dist is greater than LDA-Link when temporal granularity is small, the capability of JS-Dist is rapidly compromised for higher spatiotemporal granularity, making a much steeper drop to zero for higher granularities. Meanwhile, LDA-Link maintains a more-or-less stable performance, which proves its robustness against the sparsity of the event space.

\paragraph{Linking Views with Sparse or No Common Events}

Lastly, we assess LDA-Link's robustness to the second type of data sparsity, which is the actual degree of event overlap for a co-referent pair of sequences. In Figure~\ref{fig:overlap}, we took 10\% of the population whose views have the highest L1 distance, and plotted the best Rank-$k$ recall of LDA-Link and JS-Dist on this sample for different spatial granularities (bar graph). The plot also displays the best Rank-$k$ recall of the two algorithms on the sample of users whose profiles have no common posts at all (No Overlap). Although the overall best performance of JS-Dist is greater than LDA-Link (Figures~\ref{fig:recall},~\ref{fig:gran}), its performance is far eclipsed by LDA-Link on the spase sample and the difference is even more striking for greater input granularity. Most critically, LDA-Link is able to achieve up to 37\% Rank-$k$ recall on IG-TWT and 17\% recall on FSQ-TWT for users with no common posts at all, while JS-Dist fails to reconcile any. Again, this is the effect of LDA-Link's dimension reduction and semantic comparison.

\section{CONCLUSION}
\label{sec:concl}

We defined the problem of sequence linkage, a newly studied problem of identity uncertainty. As a solution to sparsity-robust sequence linkage, we described \emph{Split-Document} and \emph{LDA-Link}. Split-Document is a mixed-membership model for the generation of event sequences across data sets of different domains which uses the concept of motifs that account for the generation of individual events and their collective patterns. Based on this model, LDA-Link can infer the correct identity linkage structure across data sets through a semantic comparison of each sequence pair. By conducting an empirical validation in linking profiles across different location-based social media, we tested LDA-Link's robustness against factors of common event sparsity by modulating the granuarlity of the event space and testing against a selective sample of co-referent views with rare common occurrence. We proved that LDA-Link is able to significantly outperform the current state-of-the-art solution to sequence linkage when linking social media profiles that have no commonly occurring posts at all.

Split-Document can be extended to accommodate more than two data sets, each potentially having different views with duplicate identities. Extra layers of stochasticity can also be embedded into the original Split-Document model to construct more complex models. For example, one can inject an ``observation'' layer into the original model to take into consideration different rules of observation emission, which may include the probability of observation or different distortion processes (e.g. ``hit-miss'' distortion). Continuous or non-categorical variants for Gaussian or Poisson events is also a potential future direction of study. The incorporation of a Poisson model should especially be suitable for discretizing continuous time events. Another area of development is the incorporation of Bayesian non-parametric clustering models such as Dirichlet Process and Chinese Restaurant Processes as a ``clustering'' layer to model multiple duplicate identities of different views. 

\section*{Acknowledgement}

The author gratefully acknowledges Professor Augustin Chaintreau and Professor David Blei for their valuable comments and feedback.

\medskip

\bibliographystyle{unsrt}
\bibliography{ms}

\newpage

\appendix

\label{sec:analy}

\begin{center}
\textbf{\Large APPENDIX: CONVERGENCE STUDY \\}
\end{center}

In the appendix, we investigate the theoretical effectiveness of each step of the LDA-Link algorithm. Appendix~\ref{sec:topic_conv} studies the effectiveness of the topic learning phase. Appendix~\ref{sec:dim_red} investigates the proximity of the co-referent views in the semantic simplex space. Appendix~\ref{sec:k_rank_linkage} studies the optimality of the $k$-rank linkage algorithm. Appendix~\ref{sec:proofs} provides the proofs to the propositions and theorems in the appendix.

\section{LEARNING PHASE: \newline Learning Topics With and Without Omniscient Knowledge}
\label{sec:topic_conv}

The learning step in LDA-Link is equivalent to performing online variational for LDA with each view in separate domains as a single input document. In this section we lay out the steps required for proving the high-probability asymptotic proximity of the topics learned by LDA-Link to the topics learned through online LDA when every co-referent pair of views is reconciled and combined into a single document.

\subsection{SKL Divergence between the Learned Topics}

The topic-learning step is a stochastic variational inference step that optimizes the per-document ELBO $\ELBO'$ given in Equation~\ref{eq:ELBO'}. We start with the simple and slightly less realistic case that $X_d$ and $Y_{\pi(d)}$ are fetched together at each $\lambda$ iteration. $\lambda$ moves in the direction of the natural gradient of $l'_i$, which is given by
\begin{multline}
	\hat{\nabla}_\lambda l'_d =  \hat{\nabla}_\lambda\left(\sum_w X_{d,w} \sum_k \phi^X_{d,w,k}(\lambda) \Expof_\lambda[\log\beta_{k,w}]\right) \\
	+\hat{\nabla}_\lambda\left(\sum_w Y_{\pi(d),w} \sum_k \phi^Y_{\pi(d),w,k}(\lambda) \Expof_\lambda[\log\beta_{k,w}]\right) + \frac{1}{D}f(\lambda), 
\label{eq:gradient}
\end{multline}
where $\phi^X_d(\lambda)$ and $\phi^Y_{\pi(d)}(\lambda)$ are the local optimum of the variational parameters $\phi^X_d$ and $\phi^Y_{\pi(d)}$. See Appendix~\ref{sec:proof_eq_gradient} for its full derivation. LDA-Link's topic-learning algorithm applies the update formula 
\begin{equation}
	\lambda'^{(t+1)} = \lambda'^{(t)} + \rho_t \hat{\nabla}_\lambda l'_d(\lambda'^{(t)})
	\label{eq:w'}
\end{equation} 
to obtain the topic estimate $\lambda'$.

The gradient of the per-document ELBO of the Split-Document model is given by \cite{hoffman2010online}
\[
	\hat{\nabla}_\lambda l_d = \hat{\nabla}_\lambda \left(\sum_w (X_{d,w} + Y_{\pi(d),w}) \sum_k \phi_{d,w,k}(\lambda)\Expof_\lambda[\log\beta_{k,w}]\right) + \frac{1}{D}f(\lambda).
\]
The optimal update formula for the Split-Document model would thus be 
\begin{equation}
	\lambda^{(t+1)} = \lambda^{(t)} + \rho_t \hat{\nabla}_\lambda l_d(\lambda^{(t)}).
	\label{eq:w}
\end{equation}

We state the following claim that, with high probability, the SKL distance between the convergent values of the topics learned through LDA-Link (using Equation~\ref{eq:w'}) and the topics learned through the optimal update formula for the Split-Document model (using Equation~\ref{eq:w}) is bounded with high probability. 

\begin{claim}
\[
	P\left(\lim_{t\rightarrow\infty} D_{KL} (\lambda^{(t)}, \lambda'^{(t)}) < \Delta_n \right) = 1 - \epsilon_n,
\]
where $\Delta_n$ depends on the number $n$ of records in the view and $\lim_{n\rightarrow \infty} \epsilon_n=0$
\label{clm:conv}
\end{claim}

\begin{proof}[A Potential Validation Approach]
The idea is to (1) find high-probability bound on the co-referent views in the latent semantic space, (2) bound the distance between the two gradients when the co-referent views are nearby in the latent space, and then (3) to prove that the distance between the gradient descent estimates are convergent.

\begin{enumerate}
\item
Prove that there is a function $\delta$ that measures the dissimilarity of the two views in the latent semantic space, and that there is a high probability bound on a pair of co-referent views $X_d$ and $Y_{\pi(d)}$ having latent dissimilarity $\delta$ below a certain threshold that depends on the number $n$ of words in each view. That is, define a suitable choice of a dissimlarity measure $\delta$ for which there is a moderate of choice of $\epsilon_n$ and $\delta_n$ that both depend on $n$ and satisfies
\[
	P\left(\delta\left(X_d, Y_{\pi(d)}\right) < \delta_n\right) = 1 - \epsilon_n.
\]
\item
Prove the following for the topic update procedure for $\lambda'$ and $\lambda$ (Equations~\ref{eq:w'} and~\ref{eq:w}): Given a pair of co-referent views that are close to each other in the latent space, when the SKL distance between the $t$-th iteration of $\lambda$ and $\lambda'$ are within a certain threshold, then the gradients also lie close to each other for the ($t+1$)-th iteration. That is, find a suitable choice of $\delta_\nabla\left(\cdot\right)$ for measuring the difference between two gradients, for which if $\delta\left(X_d, Y_{\pi(d)}\right) < \delta_n$ and $D_{KL}\left(\lambda^{(t)}, \lambda'^{(t)}\right) < \Delta^{(t)}_n$ for some $\Delta^{(t)}_n$, then
\[
	\delta_\nabla\left(\hat{\nabla}_\lambda l_d(\lambda^{(t)}),\hat{\nabla}_\lambda l_d(\lambda'^{(t)})\right) < \delta'_n
\]
for a moderate bound $\delta'_n$ that depends on $n$.
\item
Prove that when the difference between gradients is small, the next step topic estimate is also bounded in a convergent manner. That is, if $D_{KL}\left(\lambda^{(t)}, \lambda'^{(t)}\right) < \Delta^{(t)}_n$ and $\delta_\nabla\left(\hat{\nabla}_\lambda l_d(\lambda^{(t)}),\hat{\nabla}_\lambda l_d(\lambda'^{(t)})\right) < \delta'_n$, then
\[
	D_{KL}\left(\lambda^{(t+1)}, \lambda'^{(t+1)}\right) = D_{KL}\left(\lambda^{(t)} + \rho_t \hat{\nabla}_\lambda l_d(\lambda^{(t)}),\lambda'^{(t)} + \rho_t \hat{\nabla}_\lambda l'_d(\lambda^{(t)})\right) < \Delta^{(t+1)}_n,
\]
where 
\[
	\Delta^{(t+1)}_n = \epsilon(\Delta^{(t+1)}_n, \rho_t, \delta'_n)
\] 
depends on the previous bound $\Delta^{(t)}_n$, the learning rate $\rho_t$, and the difference between gradients $\delta'_n$ in such a way that
\[
	\lim_{t\rightarrow\infty}\Delta^{(t)}_n = \Delta_n < \infty
\]
\end{enumerate}
This will prove Claim~\ref{clm:conv}.
\end{proof}

\subsection{Distance between Posterior Distributions}

The argument in the earlier paragraph provides a bound for the symmetrized KL divergence between the optimal topic estimates of the omniscient Split-Document model and its practical surrogate that considers the information geometry of the mean-field Dirichlet posterior approximation for the topic parameter $\lambda$ - namely the Independent-View model. Symmetrized KL divergence measures the distance between two topic parameters as the Jeffrey's divergence between the Dirichlet distributions that they parameterize. For the purpose of topic estimation, however, the distance of interest is not the distance between the posterior approximations, but rather the distance between the MAP estimates.

This core of thie section is Theorem~\ref{thm:dist}, which provides a bound on the JS distance between the MAP estimates of the two Dirichlet distributions in terms of their symmetrized KL divergence. This theorem is useful when, given only the symmetrized KL distance between two posterior Dirichlet distributions, it is desirable to find the bound on the JS distance between their MAP estimates (modes). 

\begin{subequations}
\begin{theorem}
Let $\mu$ and $\mu'$ each be the modes of $Dir(\eta)$ and $Dir(\eta')$. Let $C_\eta = \sum_w \eta_w - W$,  $C_{\eta'} = \sum_w \eta'_w - W$, and assume $C_\eta < C_{\eta'}$. If $\eta_w, \eta'_w > 1$ for all $w \in \mathcal{W}$, then
\[
	JS(\mu, \mu') < \frac{1}{4C_\eta} D_{KL}(\eta,\eta') + \epsilon,
\]
where $\epsilon \in O\left(\frac{C_{\eta'}}{C^2_\eta} +  \frac{C_\eta - C_{\eta'}}{C_\eta} \ln\left(\frac{C_{\eta}}{C_\eta'}\right)  \right)$ vanishes when $\frac{C_{\eta}}{C_{\eta'}}\rightarrow 1$ and $C_{\eta},C_{\eta'}\rightarrow\infty$.

\label{thm:dkl_js}
\end{theorem}
\begin{proof}
See Appendix~\ref{sec:proof_thm_dkl_js}.
\end{proof}
\end{subequations}

Using this theorem we can easily derive the following corollary, which proves that the modes of the two surrogate posterior Dirichlet distributions that are close in terms of the SKL distance are also close in JS distance.
\begin{corollary}
\label{thm:dist}
Given $2K$ Dirichlet distributions, each parameterized by $\eta_{1,...,K}$ and $\eta_{1,...,K}$ with all parameters greater than 1, we have
\[
	\sum_k JS(\mu_k, \mu'_k) < \frac{1}{4C}D_{KL}(\eta, \eta') + \sum_k\epsilon_k, 
\]
where $C = \min_{k\in[K]}\min\{\sum_w\eta_{k,w} - W, \sum_w\eta'_{kw} - W\}$ and $\epsilon_k = O\left(\frac{C_k-C'_k}{C_k}\ln\frac{C_k}{C'_k} + \frac{C'_k}{{C_k}^2}\right)$ for $C'_k = \max\left(\sum_w \eta_{k,w} - W,\sum_w \eta'_{k,w} - W\right)$ and $C_k = \min\left(\sum_w \eta_{k,w} - W,\sum_w \eta'_{k,w} - W\right)$
\end{corollary}

\section{DIMENSIONALITY REDUCTION: \newline Proximity of the Co-Referent Views in the Semantic Space}
\label{sec:dim_red}

Once the topics are learned, LDA-Link computes the optimal topic proportions for each view through a stochastic variational Bayesian approach (Appendix~\ref{sec:topic_conv}). In this section, we attempt to prove that the topic proportions of the coreferent views that are learned through the EM step in the LDA-Link algorithm are likely to be close in the simplex space with high probability. 

To achieve this we first revisit a reasonable simplification of the VB updates suggested in \cite{awasthi2015some} that will simplify our analysis. Since $\Expof_q[\log \theta_{k}] = \exp\{\psi(\gamma_k)\}$ and $\Expof_q[\log \beta_{k,w}] = \exp\{\psi(\lambda_{k,w})\}$, the iterative updates for a particular view in Algorithm~\ref{alg:score} can be rewritten as
\[
\left\{
	\begin{array}{ll}
		\phi_{w,k} &= \frac{\exp\{\psi(\gamma_k) + \psi(\lambda_{k,w})\}}{\sum_k \exp\{\psi(\gamma_k) + \psi(\lambda_{k,w})\}} \\
		\gamma_{k} &= \alpha + N \sum_w \phi_{w,k} P_{w}
	\end{array}
\right.,
\]
where we have omitted the entity index $d$ for simplicity.

Since $x \cdot e^{- \frac{1}{2x} - \frac{1}{12 x^2}} < e^{\psi(x)} < x \cdot e^{-\frac{1}{2x}}$ \cite{qi2016inequality}, we have $\lim_{x\rightarrow\infty} e^{\psi(x)} = x$. Considering this in relation to the $\gamma$ update in Algorithm~\ref{alg:score}, in large document limits where $N \rightarrow \infty$ the update equation becomes
\begin{equation}
	\theta^{t+1}_k = \sum_w P_{dw} \phi^t_{wk} \mbox{, and } 
	\phi^t_{wk} = \frac{\theta^t_k \betahat_{kw}}{\sum_k \theta^t_k \betahat_{kw}},
\label{eq:theta_update}
\end{equation}
where $\theta_k = \frac{\gamma_k}{\sum_l \gamma_k}$ and $\betahat_{kw}=\frac{\exp\{\psi(\lambda_{kw})\}}{\sum_w \exp\{\psi(\lambda_{kw})\}}$. A detailed study of this simplification and its correctness and convergence properties is presented in \cite{awasthi2015some}. We will use this approximation for the rest of this section.

The iterative procedure in Equation~\ref{eq:theta_update} converges at a point $\theta$ and $\beta$ for which 
\begin{equation}
	\theta_k = \sum_w P_{w} \phi_{wk} \mbox{, and } 
	\phi_{wk} = \frac{\theta_k \betahat_{kw}}{\sum_k \theta_k \betahat_{kw}}.
\label{eq:theta_limit}
\end{equation}
This relation implicitly defines a set $\Theta(P)$ of topic proportions $\theta$'s at which the iteration converges for some initial parameters when the empirical distribution (relative frequencies) of words is $P$. The set $\Theta(P)$ includes, but is not limited to, the global optimum of the ELBO.

We need to compute the change in $\theta$'s in $\Theta(P)$ cuased by the difference in the relative frequencies $P$. From a slight variation of Sanov's theorem\cite{unnikrishnan2015asymptotically} we get:
\begin{equation}
	\lim_{N\rightarrow\infty} -\frac{1}{N}\log P\left(JS(P, Q) \geq\lambda\right) \geq \lambda,
\end{equation}
for the relative frequencies $P$ and $Q$ of the views generated from the same distribution, so the two views are close in the simplex space with high probability. Since $\frac{1}{2}|P-Q|^2_1 \leq KL(P\|Q)$ \cite{cover2012elements},
\begin{align}
	\left|P-Q\right| &= \sqrt{2\left|P-\frac{1}{2}\left(P+Q\right)\right|^2 + 2\left|Q-\frac{1}{2}\left(P+Q\right)\right|^2} \notag\\
	&\leq \sqrt{KL\left(P,\frac{1}{2}\left(P+Q\right)\right) + KL\left(Q,\frac{1}{2}\left(P+Q\right)\right)} \notag\\
	&= \sqrt{2 JS(P,Q)},
	\label{eq:l1_JS}
\end{align}
so that if the two relative frquencies are close in the simplex space, they are also close in Euclidean space as well. This allows us to describe differential change in relative frequencies in terms of Euclidean gradients.

Consider a specific choice of $\theta_0 \in\Theta(P_0)$ for a particular empirical distribution $P_0$. We can make the following Taylor approximation to $\hat{\theta}
_0 \in \Theta(P_0 + \Delta P)$ when $P_0 + \Delta P$ is within a small neighborhood of $P_0$:
\begin{equation}
	\Delta \theta = \hat{\theta}_0 - \theta_0 \approx (\nabla \theta|_{(\theta_0, P_0)} )^T \Delta P + \frac{1}{2}\Delta P^T (\nabla^2 \theta|_{(\theta_0, P_0)}) \Delta P,
\end{equation}
where the gradient and the Hessian are computed at $(\theta, P) = (\theta_0, P_0)$.

To compute the gradient and Hessian, we must resort to implicit differentiation.

\paragraph{Frist and Second Order Necessary Conditions at Convergence}

Combining the two equations in Equation~\ref{eq:theta_update} under the limit $t \rightarrow \infty$, we obtain the following necessary conditions for the point of convergence $\theta \triangleq \lim_{t\rightarrow\infty} \theta^t$ after some rearrangement:
\begin{equation}
	\theta_k \left(1 - \sum_w \frac{P_{w} \betahat_{kw}}{\sum_k \theta_k \betahat_{kw}} \right) = 0 \mbox{ , } \forall k\in[K].
	\label{eq:conv_necc}
\end{equation}
 
Taking partial derivatives with respect to $P_v$ and setting $\eta_{kw} = \frac{\betahat_{kw}}{\sum_k \theta_k \betahat_{kw}}$, we obtain the following first order condition:
\begin{proposition}
When $\theta: K \rightarrow [0,1]$ and $P: W \rightarrow [0,1]$ are probability distributions that satisfy Equation~\ref{eq:conv_necc},
\begin{equation}
	(\nabla_v \theta_k) \left(1 - \sum_w P_w \eta_{kw}\right) + \theta_k \left(-\eta_{kv} + \sum_l (\nabla_v \theta_l) \left( \sum_w P_w \eta_{lw} \eta_{kw} \right) \right) = 0, 
	\label{eq:1st_der}
\end{equation}
where $\nabla_v (\cdot) \triangleq \frac{\pd}{\pd P_v}$.
\label{prop:1st_der}
\end{proposition}
\begin{proof}
See Appendix~\ref{sec:proof_prop_1st_der}.
\end{proof}

Note that from Equations~\ref{eq:conv_necc} and~\ref{eq:1st_der}, when $\theta_k = 0$, then $(\nabla_v \theta_k) = 0$ or $\left(1 - \sum_w P_w \eta_{kw}\right) = 0$.

Taking a second partial derivative with respect to $P_u$ we obtain the following second order condition:
\begin{proposition}
When $\theta: K \rightarrow [0,1]$ and $P: W \rightarrow [0,1]$ are probability distributions that satisfy Equation~\ref{eq:conv_necc},
\begin{align}
	(\nabla_v\theta_k)&\left\{-\eta_{ku} + \sum_l(\nabla_u\theta_l)\left(\sum_w P_w\eta_{kw}\eta_{lw}\right)\right\} + (\nabla_u\theta_k)\left\{-\eta_{kv} + \sum_l(\nabla_v\theta_l)\left(\sum_w P_w\eta_{kw}\eta_{lw}\right)\right\} \notag\\
	& + \theta_k\sum_l(\nabla_u\theta_l)\left[\eta_{kv}\eta_{lv} + \eta_{ku}\eta_{lu} - \frac{1}{2}\left\{\sum_w P_w\eta_{kw}\eta_{lw} \left(\sum_m(\nabla_u\theta_m)\eta_{mw}\right) \right\}\right] \notag\\
	& + (\nabla^2_{vu}\theta_k)\left(1-\sum_w P_w\eta_{kw}\right) + \theta_k\sum_l(\nabla^2_{vu}\theta_l)\left(\sum_w P_w\eta_{kw}\eta_{lw}\right) = 0
	\label{eq:2nd_der}
\end{align}
\label{prop:2nd_der}
\end{proposition}
\begin{proof}
See Appendix~\ref{sec:proof_prop_2nd_der}.
\end{proof}

\section{$k$-RANK LINKAGE: \newline Asymptotic Optimality of the Ranking Given the True Topics}
\label{sec:k_rank_linkage}

In this section, we provide a sketch for proving the theoretical guarantee of the correctness of LDA-Link's linking algorithm. In order to do so, we will first model the problem of finding the co-referent views as a hypothesis testing problem, in light of the approach in \cite{unnikrishnan2015asymptotically}. 

Given a particular $X$ view $X_i$, the objective is to find among all $Y$ views $Y_{j\in[D]}$ the view for which the match $(X_i,Y_j)$ is optimal. We formulate this problem as testing a set of $D$ hypotheses, each of which states that a view in $Y$ is the optimal match for $X_i$ for $D$ different $Y$ views, so that $H_j$ for $j\in[D]$ corresponds to the hypothesis that $\pi(i) = j$. Therefore, finding the correct pair of views is equivalent to finding the most optimla rule for testing the hypotheses $H=\{H_1, H_2, ...,H_D,H_R\}$, where, for the ease of analysis, we have introduced the rejection hypothesis $H_R$ as failing to find a match. Our goal is to compute the bound on the probability of error for the decision rule that links a view with the candidate view whose JS distance in the latent semantic space is minimum.

We will more formally restate the decision rule $\Omega$ designed in our algorithm. Let $\pi^*(i) = \arg\min_{j\in[D]} JS\left(\theta^X_i, \theta^Y_j\right)$, and $\pi'(i) = \arg\min_{j\in[D], j \neq \pi^*(i)} JS\left(\theta^X_i, \theta^Y_j\right)$. The decision rule $\Omega=\{\Omega_1,\Omega_2,...,\Omega_D,\Omega_R\}$, where
\[
	\Omega_{j\in[D]} = \{(X_i,Y_{1,...,D}) | \pi^*=j, \mbox{ and } JS(\theta^X_i,\theta^Y_{\pi'(i)}) \geq \lambda)\}
\]
is the acceptance region for hypothesis $J_j$, and the rejection region is
\[
	\Omega_R = \{(X,Y_{1,...,D}) | JS(\theta^X_i,\theta^Y_{\pi'(i)}) < \lambda\}
\]

The following preliminary theorem, inspired by Theorem IV.3 of \cite{unnikrishnan2015asymptotically}, may be useful for proving the error probability of $\Omega$.

\begin{theorem}
	Consider the hypothesis testing problem with the decision rule given as $\Omega$ defined above. 
	If $JS(X,Y) > \epsilon(\lambda)$ for $X,Y$ such that $JS(\theta^X,\theta^Y) > \lambda$ where $\epsilon$ is an invertible function, then 
	\[
		\lim_{n\rightarrow \infty} -\frac{1}{n}\log P_\Omega \left(error \left|\right. H_j\right) > \epsilon(\lambda),
	\]
	which indicates an exponential decay of error probabilities as a function of $n$.
\label{thm:ranking}
\end{theorem}
\begin{proof}
See Appendix~\ref{sec:proof_thm_ranking}.
\end{proof}

\paragraph{Outline of an Approach Using Chernoff Bounds}

Here we outline a different approach to finding error probability bounds using Chernoff bounds.

Consider a fixed view $X_d$, its co-referent view $Y_{\pi(d)}$, and the remaining $D-1$ views $Y_{d'\neq \pi(d)}$. A $k$-ranking error occurs if more than $k$ views among $Y_{d'\neq\pi(d)}$ has 
\[
	JS\left(\theta^X_d, \theta^Y_{d'}\right) < JS\left(\theta^X_d, \theta^Y_{\pi(d)}\right).
\]

Let us for now fix $Y_{\pi(d)}$ and let $\theta^Y_{\pi(d)} = Q$ be its topic proportion. From the graphical model shown in Figure~\ref{fig:split_doc} we see that $Y_{d'\neq \pi(d)}$'s are i.i.d. and each of them is pair-wise independent from $X_i$. Since topic proportions $\theta^X$ and $\theta^Y$ are functions of $Y$, the topic proportions $\theta^Y_{d'\neq\pi(d)}$ are also i.i.d and each of them is indepdent from $\theta^X_d$. Defining for each $d'$ a binary random variable $I_{d'}$ as
\begin{equation}
	I_{d'} = 
	\begin{cases}
		1 & \text{if } JS\left(\theta^X_d, \theta^Y_{d'}\right) < JS\left(\theta^X_d, Q\right) \\
		0 & \text{otherwise}
	\end{cases},
\end{equation}
if $Q$ is such that $I_{d'}$ satisfies $p(Q) = P\left(I_{d'}\left|\theta^Y_{\pi(d)} = Q\right.\right) \leq \frac{k}{D-1}$, we can apply the Chernoff bound to obtain the following bound on error probability:
\begin{multline}
	P\left(error\left|\theta^Y_{\pi(d)} = Q\right.\right) \\
	= P\left(\sum_{d'\neq \pi(d)} I_{d'} \geq k\right) = P\left(\frac{1}{D-1}\sum_{d'\neq \pi(d)} I_{d'} \geq \frac{k}{D-1}\right) < \exp\left(-\frac{(\delta(Q) - 1)^2 p(Q)}{2}\right)
\end{multline}
, where $\delta(Q) = \frac{k/(D-1)}{p(Q)}$. Then the total error probability becomes,
\begin{align*}
	P(error) =& \int P\left(error\left|\theta^Y_{\pi(d)} = Q\right.\right) P\left(\theta^Y_{\pi(d)} = Q\right) dQ \\
	=& \int_{p(Q) < \frac{k}{D-1}} P\left(error\left|\theta^Y_{\pi(d)} = Q\right.\right) P\left(\theta^Y_{\pi(d)} = Q\right) dQ \\
	&+ \int_{p(Q) \geq \frac{k}{D-1}} P\left(error\left|\theta^Y_{\pi(d)} = Q\right.\right) P\left(\theta^Y_{\pi(d)} = Q\right) dQ \\
	\leq& \int_{p(Q) < \frac{k}{D-1}} e^{-\frac{(\delta - 1)^2 p(Q)}{2}} P\left(\theta^Y_{\pi(d)} = Q\right) dQ \\
	&+ \int_{p(Q) \geq \frac{k}{D-1}} P\left(error\left|\theta^Y_{\pi(d)} = Q\right.\right) P\left(\theta^Y_{\pi(d)} = Q\right) dQ.
\end{align*}

If we can (1) find a closed form approximation of $p(Q)$, (2) bound the product $P\left(\theta^Y_{\pi(d)} = Q\right) P\left(error\left|\theta^Y_{\pi(d)} = Q\right.\right)$ when $p(Q) > \frac{k}{D-1}$ and (3) bound $P\left(\theta^Y_{\pi(d)} = Q\right)$ when $p(Q) \geq \frac{k}{D-1}$, then an error probability bound shall be obtainable.

\section{PROOFS AND DERIVATIONS}
\label{sec:proofs}

\subsection{Derivation of Equation~\ref{eq:gradient}}
\label{sec:proof_eq_gradient}
Following the steps in \cite{hoffman2010online}, we define a locally maximized per-document ELBO $l_i'(\lambda)$ for which we set the variational parameters $\phi^X$, $\phi^Y$, $\gamma^X$ and $\gamma^Y$ to their local optimum $\phi^X(\lambda)$, $\phi^Y(\lambda)$, $\gamma^X(\lambda)$ and $\gamma^Y(\lambda)$, so that
\[
	l_d'(\lambda) = l(X_d,\phi^X_d(\lambda),\gamma^X_d(\lambda),\lambda) + l(Y_{\pi(d)},\phi^Y_{\pi(d)}(\lambda),\gamma^Y_{\pi(d)}(\lambda),\lambda)
\]
We focus on the first term. Since the variational parameters other than $\lambda$ are set to their local optimum, 
\[
	\hat{\nabla}_\lambda\left(\phi^X_d(\lambda),\gamma^X_d(\lambda)\right) = \vec{0}.
\]
By applying chain rule we get 
\begin{align*}
	&\hat{\nabla}_\lambda l(X_d,\phi^X_d(\lambda),\gamma^X_d(\lambda),\lambda) \\
	=& \left(
			\left.
				\hat{\nabla}_\lambda l(X_d,\phi^X_d,\gamma^X_d,\lambda)
			\right|_{\phi^X_d = \phi^X_d(\lambda), \gamma^X_d = \gamma^X_d(\lambda)}
		\right) \\
	&+ \left(\hat{\nabla}_\lambda(\phi^X_d(\lambda),\gamma^X_d(\lambda))\right)^T 
		\left(
			\left. 
				\hat{\nabla}_{(\phi^X,\gamma^X)}l(X_d,\phi^X,\gamma^X,\lambda)
			\right |_{\phi^X=\phi^X_d(\lambda), \gamma^X=\gamma^X_d(\lambda)}
		\right) \\
	=& \left(
			\left.
				\hat{\nabla}_\lambda l(X_d,\phi^X_d,\gamma^X_d,\lambda)
			\right|_{\phi^X_d = \phi^X_d(\lambda), \gamma^X_d = \gamma^X_d(\lambda)}
		\right) + \vec{0},
\end{align*}
The $Y$-view term can be computed similarly. Therefore, 

\begin{align*}
	&\hat{\nabla}_\lambda l'_d = \hat{\nabla}_\lambda l(X_d,\phi^X_d(\lambda),\gamma^X_d(\lambda),\lambda)
		+ \hat{\nabla}_\lambda l(Y_{\pi(d)},\phi^Y_{\pi(d)}(\lambda),\gamma^Y_{\pi(d)}(\lambda),\lambda)\\
	=& 
		\left(
			\left.
				\hat{\nabla}_\lambda l(X_d,\phi^X_d,\gamma^X_d,\lambda)
			\right|_{\phi^X_d = \phi^X_d(\lambda), \gamma^X_d = \gamma^X_d(\lambda)}
		\right) \\
	&+
		\left(
			\left.
				\hat{\nabla}_\lambda l(Y_{\pi(d)},\phi^Y,\gamma^Y,\lambda)
			\right|_{\phi^Y=\phi^Y_{\pi(d)}(\lambda), \gamma^Y=\gamma^Y_{\pi(d)}(\lambda)} 
		\right)\\
	=&\hat{\nabla}_\lambda\left(\sum_w X_{d,w} \sum_k \phi^X_{d,w,k}(\lambda) \Expof_\lambda[\log\beta_{k,w}]\right)
	+\hat{\nabla}_\lambda\left(\sum_w Y_{{\pi(d)},w} \sum_k \phi^Y_{{\pi(d)},w,k}(\lambda) \Expof_\lambda[\log\beta_{k,w}]\right) + \frac{1}{D}f(\lambda),
\end{align*}

\subsection{Proof of Theorem~\ref{thm:dkl_js}}
\label{sec:proof_thm_dkl_js}
We start with two helper lemmas.

\begin{subequations}
\begin{lemma}
	Let $P$ and $Q$ are probability distributions over the set of $W$ words and $\supp(Q)\subseteq \supp(P)$, and let $Q^*$ be a Laplce smoothing of $Q$, $Q^*_w = \frac{Q_w + C}{1 + WC}$. Then
	\begin{equation}
		KL(P \| Q) - KL(P \| Q^*) \geq \ln\frac{1+WQ}{1+C/q},
	\label{eq:smoothing}
	\end{equation}
	where $q \triangleq \min_{w\in \supp(Q)} Q_w$. Alternatively,
	\begin{equation}
		\Expof_P \left[ \ln \frac{Q_w+C}{1+WC} \right] \leq \Expof_P \left[ \ln Q_w \right] - \ln\frac{1+WQ}{1+C/q}
	\label{eq:exp_ineq}
	\end{equation}
	\label{lem:smoothing}
\end{lemma}
\begin{proof}
	\begin{align*}
		&KL(P \| Q) - KL(P \| Q^*) =  \Expof_P \left[ \ln Q_w \right] - \Expof_P \left[ \ln \frac{Q_w+C}{1+WC} \right] \\
		=&  \Expof_P \left[\ln \frac{1 + WC}{1 + C/Q_w}\right] \geq \Expof_P \left[\ln \frac{1 + WC}{1 + C/q}\right] = \ln \frac{1 + WC}{1 + C/q} \notag
	\end{align*}
\end{proof}
\end{subequations}

Note that $q \leq \frac{1}{W}$ by definition, and the RHS of Equation~\ref{eq:smoothing} is always non-positive. The bound given in Lemma~\ref{lem:smoothing} gets closer to 0 as $q\rightarrow \frac{1}{W}$ and $C\rightarrow 0$. 

\begin{lemma}
	Let $a$ and $b$ be positive real numbers. If $a>b$, then 
	\[
		\ln\left(\frac{x+b}{x+a}\right) = O\left(\frac{1}{x}\right)
	\]
	\label{lem:big_o}
\end{lemma}
\begin{proof}
	Our goal is to find positive real numbers $M$ and $C$ that satisfy
	\[
		\left|\ln\left(\frac{x+b}{x+a}\right)\right| \leq C \left|\frac{1}{x}\right| \mbox{ for } x>M.
	\]

	For any positive real number $M > a-b$, if $x>M$ then 
	\begin{align*}
		&\left|\ln\left(\frac{x+b}{x+a}\right)\right| = \left|\ln\left(1 - \frac{a-b}{x+a}\right)\right| 
		= \sum_{n=1}^{\infty} \frac{1}{n} \left(\frac{a-b}{x+a}\right)^n < \sum_{n=1}^{\infty} \frac{(a-b)^n}{n} \frac{1}{x^n} \\
		<& \sum_{n=1}^\infty \frac{(a-b)^n}{n} \frac{1}{M^{n-1}x} = \frac{M}{x} \left( \sum_{n=1}^\infty \frac{1}{n} \left(\frac{a-b}{M}\right)^n \right) = \left|M\ln\left(1 - \frac{a-b}{M}\right)\right| \frac{1}{x}.
	\end{align*}
	Therefore, $C=\left|M\ln\left(1 - \frac{a-b}{M}\right)\right|$ satisfies the desired condition for any positive real number $M>a-b$ and the lemma holds.
\end{proof}

Now we begin the proof of Theorem~\ref{thm:dkl_js}
\begin{proof}
Let $C_\eta = \sum_w \eta_w - W$ and $C_{\eta'} = \sum_w \eta'_w - W$. Assume w.l.o.g that $C_\eta < C_{\eta'}$. The mode of $Dir(\eta)$ and $Dir(\eta')$ are $\mu= (\frac{\eta_1 - 1}{C_\eta},...,\frac{\eta_W - 1}{C_\eta})$ and $\mu' = (\frac{\eta'_1 - 1}{C_{\eta'}},...,\frac{\eta'_W - 1}{C_{\eta'}})$, respectively.
	\begin{align}
		D_{KL}(\eta,\eta') =& \sum_w (\eta_w - \eta'_w)\left(\psi(\eta_w) - \psi(\eta'_w)\right) \notag \\
		=& \sum_w (\eta_w - 1) \left(\psi(\eta_w) - \psi(\eta'_w)) + \sum_w (\eta'_w - 1) (\psi(\eta'_w) - \psi(\eta_w)\right) \notag \\
		>& \sum_w (\eta_w - 1)\left(\ln(\eta_w - 1) - \ln(\eta'_w + e^{-\gamma} - 1)\right) \notag \\
		& + \sum_w (\eta'_w - 1)\left(\ln(\eta'_w - 1) - \ln(\eta_w + e^{-\gamma} - 1)\right) \label{eq:raw}
	\end{align}
, where $\gamma$ is the Euler-Mascheroni constant. Inequality~\ref{eq:raw} follows from  $\ln(x - 1) < \psi(x) < \ln(x + e^{-\gamma} - 1)$. Let $C' = \frac{e^{-\gamma}}{C_{\eta'}}$ and consider the first term in the last inequality:
	\begin{align}
		\MoveEqLeft[3] \sum_w (\eta_w - 1)\left(\ln(\eta_w - 1) - \ln(\eta'_w + e^{-\gamma} - 1)\right) \notag \\
		=& C_\eta \sum_w \mu_w \left((\ln\mu_w + \ln C_\mu) - \left(\ln\frac{(\eta'_w - 1) + e^{-\gamma}}{C_{\eta'} + We^{-\gamma}} + \ln(C_{\eta'} + We^{-\gamma})\right)\right) \notag\\
		=& C_\eta \left(\Expof_\mu \left[\ln\mu_w\right] - \Expof_\mu\left[\frac{\mu'_w + (e^{-\gamma}/C_{\eta'})}{1 + W(e^{-\gamma}/C_{\eta'})/\mu'_{\min}}\right] + \ln C_\eta - \ln(C_{\eta'} + We^{-\gamma})\right) \notag\\
		=& C_\eta\left(\Expof_\mu \left[\ln\mu_w\right] - \Expof_\mu\left[\frac{\mu'_w + (e^{-\gamma}/C_{\eta'})}{1 + W(e^{-\gamma}/C_{\eta'})/\mu'_{\min}}\right] + \ln \left(\frac{C_\eta}{C_{\eta'}}\right) - \ln\left(1 + W\frac{e^{-\gamma}}{C_{\eta'}}\right) \right) \notag\\
		\geq& C_\eta \left( \Expof_\mu \left[\ln\mu_w\right] - \Expof_\mu\left[\ln\mu'_w\right] + \ln\left(\frac{1 + W(e^{-\gamma}/C_{\eta'})}{1 + (e^{-\gamma}/C_{\eta'})/\mu'_{\min}}\right) + \ln \left(\frac{C_\eta}{C_{\eta'}}\right) - \ln\left(1 + W\frac{e^{-\gamma}}{C_{\eta'}}\right) \right) \label{eq:kl_ineq}\\
		=& C_\eta \left( KL\left(\mu \| \mu'\right) + \ln \left(\frac{C_\eta}{C_{\eta'}}\right) - \ln\left(1 + W\frac{e^{-\gamma}}{C_{\eta'}}\right) \right) \notag\\
		>& C_\eta \left( KL\left(\mu \| \mu'\right) + \ln \left(\frac{C_\eta}{C_{\eta'}}\right) - W\frac{e^{-\gamma}}{C_{\eta'}})\right) \label{eq:ln_ineq}  \\
		=& C_\eta KL\left(\mu \| \mu' \right) + C_\eta\ln\left(\frac{C_{\eta'} + We^{-\gamma}}{C_{\eta'} + e^{-\gamma}/\mu'_{\min}}\right) + C_\eta \ln\left(\frac{C_\eta}{C_{\eta'}}\right) - We^{-\gamma}\frac{C_\eta}{C_{\eta'}} \notag
	\end{align}
Inequalities~\ref{eq:kl_ineq} and ~\ref{eq:ln_ineq} follow from Inequality~\ref{eq:exp_ineq} of Lemma~\ref{lem:smoothing} and $\ln(x+1) < x$ for $x>0$. Combining these results,
	\begin{align}
		D_{KL}(\eta,\eta') >& 
			\left(C_\eta KL\left(\mu \| \mu' \right) + C_\eta\ln\left(\frac{C_{\eta'} + We^{-\gamma}}{C_{\eta'} + e^{-\gamma}/\mu'_{\min}}\right) + C_\eta \ln\left(\frac{C_\eta}{C_{\eta'}}\right) - We^{-\gamma}\frac{C_\eta}{C_{\eta'}}\right) \notag\\
			&+ \left(C_{\eta'} KL\left(\mu' \| \mu \right) + C_{\eta'}\ln\left(\frac{C_{\eta} + We^{-\gamma}}{C_{\eta} + e^{-\gamma}/\mu_{\min}}\right) + C_{\eta'} \ln\left(\frac{C_{\eta'}}{C_\eta}\right) - We^{-\gamma}\frac{C_{\eta'}}{C_\eta}\right)\notag\\	
			=& C_\eta KL\left(\mu \| \mu'\right) + C_{\eta'} KL\left(\mu' \| \mu\right) + \epsilon(C_\eta, C_{\eta'}, \mu_{\min}, \mu'_{\min})\notag\\
			\geq& C_\eta \left\{ J(\mu, \mu') + \epsilon(C_\eta,C_{\eta'},\mu_{\min},\mu'_{\min}) \right\} \label{eq:kl_and_j}
	\end{align}
where
	\begin{align}
		&\epsilon(C_\eta,C_{\eta'},\mu_{\min},\mu'_{\min}) \notag \\
		=& \ln\left(\frac{C_{\eta'} + We^{-\gamma}}{C_{\eta'} + e^{-\gamma}/\mu'_{\min}}\right) 
		+\frac{C_{\eta'}}{C_\eta} \ln\left(\frac{C_{\eta} + We^{-\gamma}}{C_{\eta} + e^{-\gamma}/\mu_{\min}}\right) 
		+ \frac{C_\eta - C_{\eta'}}{C_\eta} \ln\left(\frac{C_{\eta}}{C_\eta'}\right) 
		- We^{-\gamma}\left(\frac{1}{C_\eta'} + \frac{C_{\eta'}}{C^2_\eta}\right) \notag\\
		=& O\left(\frac{C_{\eta'}}{C^2_\eta} +  \frac{C_\eta - C_{\eta'}}{C_\eta} \ln\left(\frac{C_{\eta}}{C_\eta'}\right)  \right).
	\label{eq:big_o}
	\end{align}
Equation~\ref{eq:big_o} follows from Lemma~\ref{lem:big_o}. Rearranging Equation~\ref{eq:kl_and_j} we get,
	\[
		J(\mu,\mu') < \frac{1}{C_\eta} D_{KL}(\eta,\eta') + O\left(\frac{C_{\eta'}}{C^2_\eta} +  \frac{C_\eta - C_{\eta'}}{C_\eta} \ln\left(\frac{C_{\eta}}{C_\eta'}\right)  \right),
	\]
and since $JS(\mu,\mu') \leq \frac{1}{4}J(\mu,\mu')$\cite{lin1991divergence}, 
	\[
		JS(\mu, \mu') \leq \frac{1}{4C_\eta} D_{KL}(\eta,\eta') + O\left(\frac{C_{\eta'}}{C^2_\eta} +  \frac{C_\eta - C_{\eta'}}{C_\eta} \ln\left(\frac{C_{\eta}}{C_\eta'}\right)  \right)
	\]
The remainder approaches 0 as $C_\eta, C_{\eta'} \rightarrow \infty$ and $\frac{C_{\eta'}}{C_\eta} \rightarrow 1$
\end{proof}

\subsection{Proof of Proposition~\ref{prop:1st_der}}
\label{sec:proof_prop_1st_der}
First we state a trivial observation, the proof of which is straightforward:
\begin{equation}
	\frac{\pd \eta_{kw}}{\pd \theta_l} = \frac{\pd}{\pd \theta_l}\frac{\betahat_{kw}}{\sum_k \theta_k \betahat_{kw}} = -\frac{\betahat_{kw} \betahat_{lw}}{(\sum_{k'} \theta_{k'} \betahat_{k'w})^2} = -\eta_{k'w}\eta_{lw}
\end{equation}
Therefore, we have 
\begin{equation}
	\frac{\pd \eta_{kw}}{\pd P_v} = \sum_l \frac{\pd \theta_l}{\pd P_v}\frac{\pd \eta_{kw}}{\pd \theta_l} = -\sum_l \left(\nabla_v \theta_l\right)\eta_{kw}\eta_{lw},
\label{eq:eta_der}
\end{equation}
which will be used later in the proof.

Implicitly differentiating Equation~\ref{eq:theta_limit} with respect to $P_w$,
\begin{align*}
	&\frac{\pd}{\pd P_v} \left\{\theta_k \left(1 - \sum_w \frac{P_{w} \betahat_{kw}}{\sum_k \theta_k \betahat_{kw}} \right) \right\} \\
	&= \frac{\pd}{\pd P_v} \left\{\theta_k \left(1 - \sum_w P_w \eta_{kw}\right)\right\} \\
	&= \left(\nabla_v \theta_k\right)\left(1 - \sum_w P_w \eta_{kw}\right) + \theta_k \left\{- \sum_w \left(\frac{\pd P_w}{\pd P_v}\eta_{kw} + P_w\frac{\pd \eta_{kw}}{\pd P_v}\right) \right\} \\
	&= \left(\nabla_v \theta_k\right)\left(1 - \sum_w P_w \eta_{kw}\right) + \theta_k \left\{ -\eta_{kv} - \sum_w P_w \left(-\sum_l \left(\nabla_v \theta_l\right)\eta_{kw}\eta_{lw}\right) \right\} \\
	&= \left(\nabla_v \theta_k\right)\left(1 - \sum_w P_w \eta_{kw}\right) + \theta_k \left\{ -\eta_{kv} + \sum_l \left(\nabla_v \theta_l\right)\left(\sum_w P_w\eta_{lw}\eta_{kw}\right) \right\},
\end{align*}
whereby obtaining Equation~\ref{eq:1st_der} in Proposition~\ref{prop:1st_der}

\subsection{Proof of Proposition~\ref{prop:2nd_der}}
\label{sec:proof_prop_2nd_der}
\begin{proof}
From Equation~\ref{eq:eta_der}, we get the following:
\begin{equation}
	\frac{\pd\eta_{kw}}{\pd P_u} = \sum_l \frac{\pd\theta_l}{\pd P_u}\frac{\pd\eta_{kw}}{\pd\theta_l} = \sum_l\left(\nabla_u\theta_l\right)\left(-\eta_{kw}\eta_{lw}\right) = \eta_{kw}\sum_l\left(\nabla_u\theta_l\right)\eta_{lw}
	\label{eq:diff_eta}
\end{equation}
which will be useful later in the proof.

Differentiating Equation~\ref{eq:1st_der} with respect to $P_u$,
\begin{equation}
	\nabla^2_{vu}\theta_k = 
	\nabla_u
	\left\{
		\underbrace{(\nabla_v \theta_k) \left(1 - \sum_w P_w \eta_{kw}\right)}_{\text{A}}
		+ 
		\underbrace{\theta_k \left(-\eta_{kv} + \sum_l (\nabla_v \theta_l) \left( \sum_w P_w \eta_{lw} \eta_{kw} \right) \right)}_{\text{B}}
	\right\}
\end{equation}
We differentiate A and B individually, each using chain rule.
\paragraph{Differentiating A}
\begin{align}
	&\nabla_u\left\{(\nabla_v \theta_k) \left(1 - \sum_w P_w \eta_{kw}\right)\right\} \notag\\
	&= \left(\nabla^2_{vu}\theta_k\right)\left(1-\sum_wPw\eta_{kw}\right) + \nabla_v\theta_k\left\{-\eta_{ku} - \sum_w P_w\frac{\pd\eta_{kw}}{\pd P_u}\right\} \notag\\
	&= \left(\nabla^2_{vu}\theta_k\right)\left(1-\sum_wPw\eta_{kw}\right) + \nabla_v\theta_k\left\{-\eta_{ku} - \sum_w P_w\left(-\sum_l\left(\nabla_u\theta_l\right)\eta_{kw}\eta_{lw}\right)\right\} \notag\\
	&= \left(\nabla^2_{vu}\theta_k\right)\left(1-\sum_wPw\eta_{kw}\right) + \nabla_v\theta_k\left\{-\eta_{ku} + \sum_w P_w\eta_{kw}\left(\sum_l\left(\nabla_u\theta_l\right)\eta_{lw}\right)\right\}
	\label{eq:diff_A}
\end{align}
\paragraph{Differentiating B}
\begin{align}
	&\nabla_u\left\{\theta_k \left(-\eta_{kv} + \sum_l (\nabla_v \theta_l) \left( \sum_w P_w \eta_{lw} \eta_{kw} \right) \right)\right\} \notag \\
	&= \left(\nabla_u\theta_k\right)\left(-\eta_{kv} + \sum_l (\nabla_v \theta_l) \left( \sum_w P_w \eta_{lw} \eta_{kw} \right) \right)
	+
	\theta_k\left\{-\frac{\pd\eta_{kv}}{\pd P_u} + \sum_l\left\{
		\left(\nabla_v\theta_l\right)\sum_w\frac{\pd}{\pd P_u}\left(P_w\eta_{kw}\eta_{lw}\right)
	\right\}\right\} \notag\\
	&= \left(\nabla_u\theta_k\right)\left(-\eta_{kv} + \sum_l (\nabla_v \theta_l) \left( \sum_w P_w \eta_{lw} \eta_{kw} \right) \right)
	+ \theta_k\eta_{kv}\sum_l\left(\nabla_u\theta_l\right)\eta_{lv} \notag\\
	&+ \theta_k\sum_l\left\{
		\left(\nabla^2_{vu}\theta_l\right)\sum_w P_w\eta_{kw}\eta_{lw}
		+ \left(\nabla_v\theta_l\right)\left\{
			\eta_{ku}\eta_{lu} - 2\sum_w P_w\eta_{kw}\eta_{lw}\left(
				\sum_w\left(\nabla_u\theta_m\right)\eta_{mw}
			\right)
		\right\}
	\right\},
	\label{eq:diff_B}
\end{align}
where in deriving Equation~\ref{eq:diff_B} we hafe used Equation~\ref{eq:diff_eta}. Combining Equations~\ref{eq:diff_A} and~\ref{eq:diff_B}, we obtain Equation~\ref{eq:2nd_der} in Proposition~\ref{prop:2nd_der}.
\end{proof}

\subsection{Proof of Theorem~\ref{thm:ranking}}
\label{sec:proof_thm_ranking}
First we provide two helper lemmas without proof. These lemmas and their proofs were introduced in \cite{cover2012elements}.

\begin{lemma}\emph{\cite{cover2012elements}}
	For a sequence $s\in Z^n$ and any probability distribution $Q$,
	\[
		Q(s) \leq 2^{-nH(P^s)}
	\]
	where $P_s$ is the relative frequency of the sequence $s$.
	\label{lem:prob_ent}
\end{lemma}

\begin{lemma}\emph{\cite{cover2012elements}}
	For a finite alphabet $\mathcal{W}$ for which $|\mathcal{W}| = W$, 
	\[
		\sum_{x\in W^n} 2^{-n(H(P_x))} \leq (n+1)^W
	\]
	where $P_x$ is the relative frequency vector (or ``type'' in information theory) of $x$, and $H(P)$ is the Shannonn entropy of $P$.
	\label{lem:ent_sum_bound}
\end{lemma}

We now begin the proof of the theorem.
\begin{proof}
	Define a slight modification of $\Omega$, namely
	\[
		\tilde{\Omega}_j = \{(X_i,Y_{1,...,n}) | JS(\theta^X_i, \theta^Y_j) \geq \lambda\}.
	\]
	Then,
	\[
		\Omega_k \subset \tilde{\Omega}_j \forall k \neq j,
	\]
	and therefore
	\[
		\cup_{k\neq j} \Omega_k \subset \cup_{k \neq j} \left(\cap_{l \neq k} \tilde{\Omega}_l\right) \subset \tilde{\Omega}_j
	\]
	The probability of error then becomes
	\begin{align*}
		P_\Omega(error | H_j) =& \sum_{\cup_{k \neq j} \Omega_k} p(X_i, Y_1,...,Y_D) \\
		=& \sum_{\cup_{k \neq j} \Omega_k} p(X_i, Y_k) \prod_{l \neq k} p(Y_l) \\
		\leq& \sum_{\tilde{\Omega}_j} p(X_i, Y_j) \prod_{l \neq j} p(Y_l).
	\end{align*}
	Since $X_i$ and $Y_k$ are i.i.d. sequences of length $n$, one can consider a combined sequence of length $2n$ and apply Lemma~\ref{lem:prob_ent} to see that 
	\[
		P(X_i, Y_k) \leq 2^{-2n\left(H\left(\frac{1}{2}(X_i + Y_k)\right)\right)} \leq 2^{-n\left(H(X_i) + H(Y_k) + JS(X_i, Y_k)\right)}.
	\]
	Thus,
	\begin{align*}
		P_\Omega(error | H_j) \leq& \sum_{\tilde{\Omega}_j} 2^{-2n\left(H(X_i) + H(Y_k) + JS(X_i, Y_k)\right)} \prod_{l \neq k} 2^{H(Y_l)} \\
		=& \sum_{\tilde{\Omega}_j} 2^{-n(JS(X_i, Y_j))} 2^{-n(H(X_i) + \sum_l H(Y_l)))}  \\
		<& \sum_{\tilde{\Omega}_j} 2^{-n\epsilon(\lambda)} 2^{-n(H(X_i) + \sum_l H(Y_l)))} \\
		=& 2^{-n\epsilon(\lambda)} \sum_{\tilde{\Omega}_j} 2^{-n(H(X_i) + \sum_l H(Y_l)))} \\
		<& 2^{-n\epsilon(\lambda)} \sum_{\mathcal{W}^{D+1}} 2^{-n(H(X_i) + \sum_l H(Y_l)))}
	\end{align*}
	From Lemma~\ref{lem:ent_sum_bound} it follows that
	\begin{align*}
		P_\Omega(error | H_j)  <& 2^{-n\epsilon(\lambda)}\left((n+1)^W\right)^{D+1} \\
		=& 2^{-n\epsilon(\lambda)}(n+1)^{WD+W} \\
		<& 2^{-n \epsilon(\lambda)}\cdot 2^{\log(n+1)\cdot W(D+1)}
	\end{align*}
	and thus,
	\[
		-\frac{1}{n}\log P_\Omega\left(error \left|\right. H_j \right) > \epsilon(\lambda) - W(D+1) \frac{\log(n+1)}{n}.
	\]
	
	In the limit $n\rightarrow\infty$, we obtain Theorem~\ref{thm:ranking}.
\end{proof}

\end{document}